\documentclass[12pt]{article}
\RequirePackage{amsthm,amsmath,amsbsy,amsfonts}
\usepackage{bm,graphicx,psfrag,epsf}
\usepackage{enumerate}
\usepackage{natbib}
\usepackage{paralist, booktabs, multirow}

\newcommand{\blind}{0}

\numberwithin{equation}{section}
\theoremstyle{plain}

\usepackage{amssymb}
\usepackage[usenames]{color}

\usepackage[mathscr]{eucal}

\usepackage{stmaryrd}

\usepackage{xspace}

\usepackage{paralist}

\usepackage{ifpdf}

\usepackage{graphicx}
\ifpdf
\usepackage{epstopdf}
\DeclareGraphicsExtensions{.eps,.pdf}
\else
\usepackage{epsfig}
\fi

\usepackage{subfig}

\ifpdf
\usepackage[colorlinks,urlcolor=blue,citecolor=blue,linkcolor=blue]{hyperref}
\else
\newcommand{\href}[2]{{#2}}
\usepackage{url}
\fi

\usepackage{tabulary}

\usepackage{algorithm}
\usepackage{algpseudocode}

\ifpdf
\newcommand{\Sec}[1]{\hyperref[sec:#1]{\S\ref*{sec:#1}}} 
\newcommand{\App}[1]{\hyperref[sec:#1]{Appendix~\ref*{sec:#1}}} 
\newcommand{\Eqn}[1]{\hyperref[eq:#1]{{\rm (\ref*{eq:#1})}}} 
\newcommand{\Part}[1]{\hyperref[part:#1]{(\ref*{part:#1})}} 
\newcommand{\Fig}[1]{\hyperref[fig:#1]{Figure~\ref*{fig:#1}}} 
\newcommand{\Tab}[1]{\hyperref[tab:#1]{Table~\ref*{tab:#1}}} 
\newcommand{\Thm}[1]{\hyperref[thm:#1]{Theorem~\ref*{thm:#1}}} 
\newcommand{\Lem}[1]{\hyperref[lem:#1]{Lemma~\ref*{lem:#1}}} 
\newcommand{\Prop}[1]{\hyperref[prop:#1]{Proposition~\ref*{prop:#1}}} 
\newcommand{\Cor}[1]{\hyperref[cor:#1]{Corollary~\ref*{cor:#1}}} 
\newcommand{\Def}[1]{\hyperref[def:#1]{Definition~\ref*{def:#1}}} 
\newcommand{\Alg}[1]{\hyperref[alg:#1]{Algorithm~\ref*{alg:#1}}} 
\newcommand{\Ex}[1]{\hyperref[ex:#1]{Example~\ref*{ex:#1}}} 
\newcommand{\As}[1]{\hyperref[as:#1]{Assumption~{\rm\ref*{as:#1}}}} 
\newcommand{\Reg}[1]{\hyperref[as:#1]{Condition~\ref*{reg:#1}}} 
\newcommand{\AlgLine}[2]{\hyperref[alg:#1]{line~\ref*{line:#2} of Algorithm~\ref*{alg:#1}}}
\newcommand{\AlgLines}[3]{\hyperref[alg:#1]{lines~\ref*{line:#2}--\ref*{line:#3} of Algorithm~\ref*{alg:#1}}}
\else
\newcommand{\Sec}[1]{{\S\ref{sec:#1}}} 
\newcommand{\App}[1]{{Appendix~\ref{sec:#1}}} 
\newcommand{\Eqn}[1]{{(\ref{eq:#1})}} 
\newcommand{\Part}[1]{{(\ref{part:#1})}} 
\newcommand{\Fig}[1]{{Figure~\ref{fig:#1}}} 
\newcommand{\Tab}[1]{{Table~\ref{tab:#1}}} 
\newcommand{\Thm}[1]{{Theorem~\ref{thm:#1}}} 
\newcommand{\Lem}[1]{{Lemma~\ref{lem:#1}}} 
\newcommand{\Prop}[1]{{Property~\ref{prop:#1}}} 
\newcommand{\Cor}[1]{{Corollary~\ref{cor:#1}}} 
\newcommand{\Def}[1]{{Definition~\ref{def:#1}}} 
\newcommand{\Alg}[1]{{Algorithm~\ref{alg:#1}}} 
\newcommand{\Ex}[1]{{Example~\ref{ex:#1}}} 
\newcommand{\As}[1]{{Assumption~\ref{as:#1}}} 
\newcommand{\Reg}[1]{{R~\ref{reg:#1}}} 
\newcommand{\AlgLine}[2]{{line~\ref{line:#2} of Algorithm~\ref{alg:#1}}}
\newcommand{\AlgLines}[3]{{lines~\ref{line:#2}--\ref{line:#3} of Algorithm~\ref{alg:#1}}}
\fi

\usepackage{paralist, booktabs, multirow}

\newtheorem{assumption}{Assumption}[section]

\usepackage{mathtools}

\usepackage{algorithm}
\usepackage{algpseudocode}
\usepackage{tikz}
\usetikzlibrary{arrows}
\newtheorem{proposition}{Proposition}[section]

\newcommand{\qtext}[1]{\quad\text{#1}\quad}

\newcommand{\prox}{\mathop{\rm prox}\nolimits}
\newcommand{\amp}{\mathop{\:\:\,}\nolimits}

\newcommand{\bc}{\bm{\mathbf{c}}}

\newcommand{\be}{\bm{\mathbf{e}}}

\newcommand{\bg}{\bm{\mathbf{g}}}

\newcommand{\br}{\bm{\mathbf{r}}}

\newcommand{\bu}{\bm{\mathbf{u}}}
\newcommand{\bv}{\bm{\mathbf{v}}}

\newcommand{\btildev}{\bm{\tilde{\mathbf{v}}}}
\newcommand{\bw}{\bm{\mathbf{w}}}
\newcommand{\bx}{\bm{\mathbf{x}}}
\newcommand{\by}{\bm{\mathbf{y}}}
\newcommand{\bz}{\bm{\mathbf{z}}}
\newcommand{\bA}{\bm{\mathbf{A}}}
\newcommand{\bB}{\bm{\mathbf{B}}}

\newcommand{\bI}{\bm{\mathbf{I}}}

\newcommand{\bL}{\bm{\mathbf{L}}}
\newcommand{\bM}{\bm{\mathbf{M}}}

\newcommand{\bS}{\bm{\mathbf{S}}}
\newcommand{\bT}{\bm{\mathbf{T}}}
\newcommand{\bU}{\bm{\mathbf{U}}}
\newcommand{\bV}{\bm{\mathbf{V}}}
\newcommand{\btildeU}{\bm{\tilde{\mathbf{U}}}}

\newcommand{\btildeV}{\bm{\tilde{\mathbf{V}}}}

\newcommand{\bX}{\bm{\mathbf{X}}}
\newcommand{\bY}{\bm{\mathbf{Y}}}
\newcommand{\bZ}{\bm{\mathbf{Z}}}
\newcommand{\bOne}{\bm{\mathbf{1}}}
\newcommand{\bZero}{\bm{\mathbf{0}}}

\newcommand{\blambda}{\bm{\mathbf{\lambda}}}

\newcommand{\btildelambda}{\bm{\tilde{\mathbf{\lambda}}}}

\newcommand{\bDelta}{\bm{\mathbf{\Delta}}}

\newcommand{\bLambda}{\bm{\mathbf{\Lambda}}}

\newcommand{\bPhi}{\bm{\mathbf{\Phi}}}

\newcommand{\Kron}{\otimes} 

\newcommand{\Real}{\mathbb{R}}

\begin{document}

\def\spacingset#1{\renewcommand{\baselinestretch}%
{#1}\small\normalsize} \spacingset{1}


\if0\blind
{
  \title{\bf Splitting Methods for Convex Clustering}
  \author{Eric C. Chi\thanks{
    Eric C. Chi (E-mail: echi@rice.edu) is Postdoctoral Research Associate, Department of Electrical and Computer Engineering, Rice University, Houston, TX 77005.} \,and
    Kenneth Lange\thanks{
    Kenneth Lange (E-mail: klange@ucla.edu) is Professor, Departments of Biomathematics, Human Genetics, and Statistics, University of California, Los Angeles, CA 90095-7088.}\\}
    \date{}
  \maketitle
} \fi

\if1\blind
{
  \bigskip
  \bigskip
  \bigskip
  \begin{center}
    {\LARGE\bf Title}
\end{center}
  \medskip
} \fi

\bigskip
\begin{abstract}
Clustering is a fundamental problem in many scientific applications. Standard methods such as $k$-means, Gaussian mixture models, and hierarchical clustering, however, are beset by local minima, which are sometimes drastically suboptimal. Recently introduced convex relaxations of $k$-means and hierarchical clustering shrink cluster centroids toward one another and ensure a unique global minimizer. In this work we present two splitting methods for solving the convex clustering problem. The first is an instance of the alternating direction method of multipliers (ADMM); the second is an instance of the alternating minimization algorithm (AMA). In contrast to previously considered algorithms, our ADMM and AMA formulations provide simple and unified frameworks for solving the convex clustering problem under the previously studied norms and open the door to potentially novel norms. We demonstrate the performance of our algorithm on both simulated and real data examples. While the differences between the two algorithms appear to be minor on the surface, complexity analysis and numerical experiments show AMA to be significantly more efficient.
\end{abstract}

\noindent%

\spacingset{1.45}
\section{Introduction}
\label{sec:introduction}
In recent years convex relaxations of many fundamental, yet combinatorially hard, optimization problems 
in engineering, applied mathematics, and statistics have been introduced \citep{Tro2006}. Good, and sometimes nearly optimal solutions, can be achieved at affordable computational prices for problems that appear at first blush to be computationally intractable. In this paper, we introduce two new algorithmic frameworks based on variable splitting that generalize and extend recent efforts to convexify the classic unsupervised problem of clustering.

\citet{LinOhlLju2011} and \citet{HocVerBac2011} formulate the clustering task as a convex optimization problem. Given $n$ points $\bx_1,\ldots,\bx_n$ in $\Real^p$, they suggest minimizing the convex criterion
\begin{eqnarray}
F_{\gamma}(\bU) & = & \frac{1}{2}\sum_{i=1}^n \|\bx_i-\bu_i\|_2^2 + \gamma \sum_{i<j}w_{ij} \|\bu_i-\bu_j \|,
\label{eq:objective_function}
\end{eqnarray}
where $\gamma$ is a positive tuning constant, $w_{ij}$ is a nonnegative weight, and the $i$th column
$\bu_i$ of the matrix $\bU$ is the cluster center attached to point $\bx_i$. \citet{LinOhlLju2011} consider an $\ell_p$ norm penalty on the differences $\bu_i - \bu_j$ while \citet{HocVerBac2011} consider $\ell_1$, $\ell_2,$ and $\ell_\infty$ penalties. In the current paper, an arbitrary norm defines the penalty.

The objective function bears some similarity to the fused lasso signal approximator \citep{TibSauRos2005}. When the $\ell_1$ penalty is used in definition (\ref{eq:objective_function}), we recover a special case of the General Fused Lasso \citep{Hoe2010,TibTay2011}. In the graphical interpretation of clustering, each point corresponds to a node in a graph, and an edge connects nodes $i$ and $j$ whenever $w_{ij} > 0$. \Fig{graph} depicts an example. In this case, the objective function $F_{\gamma}(\bU)$ separates over the connected components of the underlying graph. Thus, one can solve for the optimal $\bU$ component by component. Without loss of generality, we assume the graph is connected. 

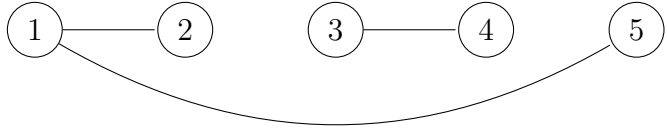
\begin{figure}
\centering
\begin{tikzpicture}[>=stealth',shorten >=1pt,auto,node distance=2cm,
  main node/.style={circle,draw}]

  \node[main node] (1) {1};
  \node[main node] (2) [right of=1] {2};
  \node[main node] (3) [right of=2] {3};
  \node[main node] (4) [right of=3] {4};
  \node[main node] (5) [right of=4] {5};

  \path[every node/.style={font=\sffamily\small}]
    (1)         edge [bend right] node[left] {} (5)
    edge node [left] {} (2)
    (3) edge node [right] {} (4);
\end{tikzpicture}
\caption{A graph with positive weights $w_{12}$, $w_{15}$, $w_{34}$ and all other weights $w_{ij} = 0$.}
\label{fig:graph}
\end{figure}

When $\gamma=0$, the minimum is attained when $\bu_i=\bx_i$, and each point occupies a unique cluster.
As $\gamma$ increases, the cluster centers begin to coalesce. Two points $\bx_i$ and $\bx_j$ with $\bu_i=\bu_j$ are said to belong to the same cluster. For sufficiently high $\gamma$ all points coalesce into a single cluster. Because the objective function $F_{\gamma}(\bU)$ in equation \Eqn{objective_function} is strictly convex and coercive, it possesses a unique minimum point for each value of $\gamma$.  If we plot the solution matrix $\bU$ as a function of $\gamma$, then we can ordinarily identify those values of $\gamma$ giving $k$ clusters for any integer $k$ between $p$ and $1$. In theory, $k$ can decrement by more than 1 as certain critical values of $\gamma$ are passed. Indeed, when points are not well separated, we observe that many centroids will coalesce abruptly unless care is taken in choosing the weights $w_{ij}$.

The benefits of this formulation are manifold. As we will show, convex relaxation admits a simple and fast iterative algorithm that is guaranteed to converge to the unique global minimizer. In contrast, the classic $k$-means problem has been shown to be NP-hard \citep{AloDesHan2009,DasFre2009}. In addition, the classical greedy algorithm for solving $k$-means clustering often gets trapped in suboptimal local minima \citep{For1965,Llo1982, Mac1967}.  

\begin{figure}
\centering
\includegraphics[scale=0.35]{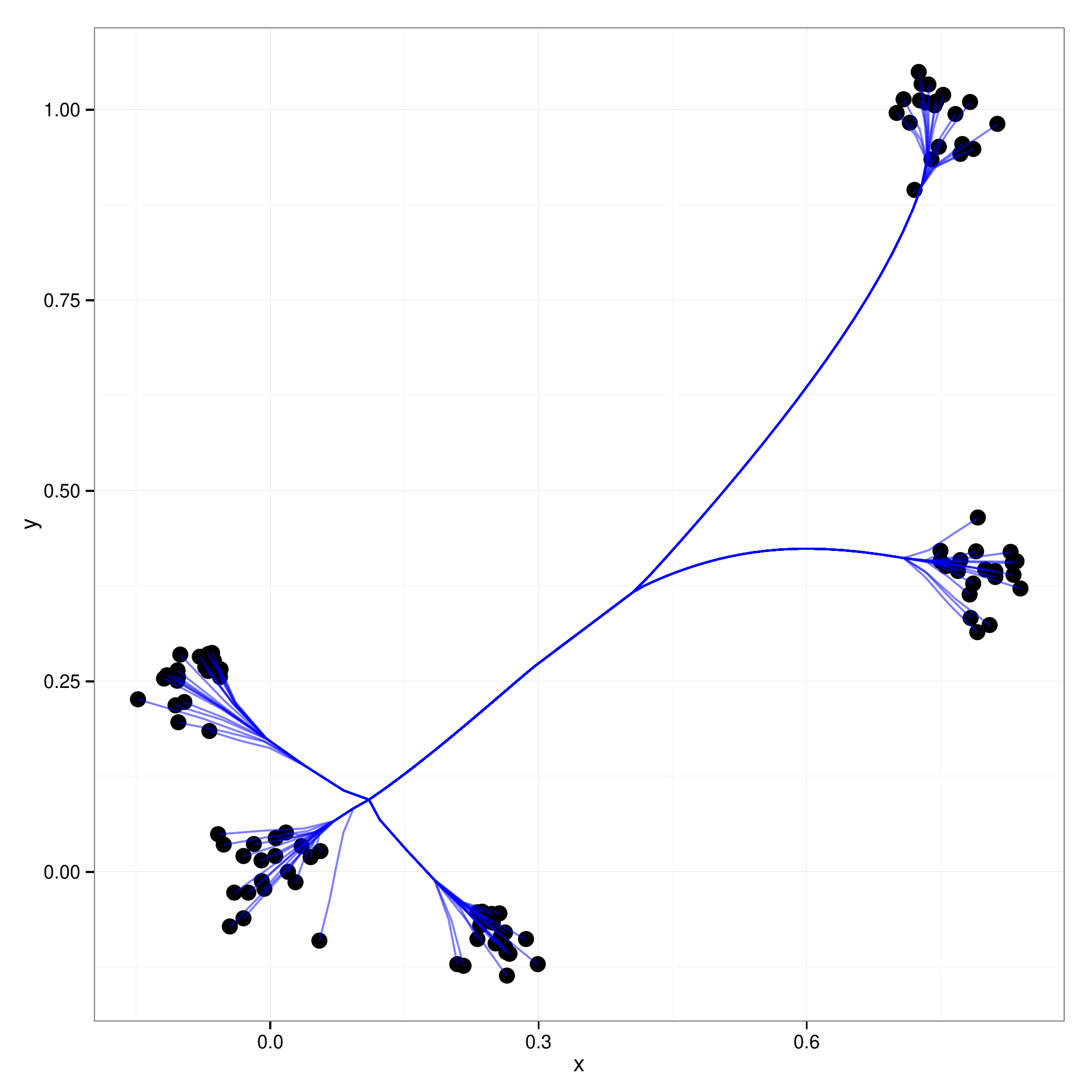}
\caption{Cluster path assignment: The simulated example shows five well separated clusters and the assigned clustering from applying the convex clustering algorithm using an $\ell_2$-norm. The lines trace the path of the individual cluster centers as the regularization parameter $\gamma$ increases.}
\label{fig:clusterpath}
\end{figure}

Another vexing issue in clustering is determining the number of clusters. Agglomerative hierarchical clustering \citep{GowRos1969,Joh1967,LanWil1967,Mur1983,War1963} finesses the problem by computing an entire clustering path. Agglomerative approaches, however, can be computationally demanding and tend to fall into suboptimal local minima since coalescence events are not reversed. The alternative convex relaxation considered here performs continuous clustering just as the lasso \citep{CheDonSau1998, Tib1996} performs continuous variable selection. \Fig{clusterpath} shows how the solutions to the alternative convex problem traces out an intuitively appealing, globally optimal, and computationally tractable solution path.

\subsection{Contributions}
Our main contributions are two new methods for solving the convex relaxation and their application to clustered regression problems. Relatively little work has been published on algorithms for solving this optimization problem. In fact, the only other paper introducing dedicated algorithms for minimizing criterion (\ref{eq:objective_function}) that we are aware of is \cite{HocVerBac2011}. \citet{LinOhlLju2011} used the off-the-shelf convex solver CVX \citep{CVX2012,GraBoy2008} to generate solution paths. \citet{HocVerBac2011} note that CVX is useful for solving small problems but a dedicated formulation is required for scalability. Thus, they introduced three distinct algorithms for the three most commonly encountered norms. Given the $\ell_1$ norm and unit weights $w_{ij}$, the objective function separates, and they solve the convex clustering problem by the exact path following method designed for the fused lasso \citep{Hoe2010}. For the $\ell_1$ and $\ell_2$ norms with arbitrary weights $w_{ij}$, they employ subgradient descent in conjunction with active sets. Finally, they solve the convex clustering problem under the $\ell_\infty$ norm by viewing it as minimization of a Frobenius norm over a polytope. In this guise, the problem succumbs to the Frank-Wolfe algorithm \citep{FraWol1956} of quadratic programming.

In contrast to this piecemeal approach, we introduce two similar generic frameworks for minimizing the convex clustering objective function with an arbitrary norm. One approach solves the problem by the alternating direction method of multipliers (ADMM), while the other solves it by the alternating minimization algorithm (AMA). The key step in both cases computes the proximal map of a given norm. Consequently, both of our algorithms apply provided the penalty norm admits efficient computation of its proximal map. 

In addition to introducing new algorithms for solving the convex clustering problem, the current paper contributes 
in other concrete ways: (a) We combine existing results on AMA and ADMM with the special structure of the convex clustering problem to characterize both of the new algorithms theoretically. In particular, the clustering problem formulation gives a minimal set of extra assumptions needed to prove the convergence of the ADMM iterates to the unique global minimum. We also explicitly show how the computational and storage complexity of our algorithms scales with the connectivity of the underlying graph. Examination of the dual problem enables us to identify a fixed step size for AMA that is associated with the Laplacian matrix of the underlying graph. Finally, our complexity analysis enables us to rigorously quantify the efficiency of the two algorithms so the two methods can be compared. (b) We provide new proofs of intuitive properties of the solution path. These results are tied solely to the minimization of the objective function \Eqn{objective_function} and hold  regardless of the algorithm used to find the minimum point. (c) We provide guidance on how to choose the weights $w_{ij}$. Our suggested choices diminish computational complexity and enhance solution quality. In particular, we show that employing $k$-nearest neighbor weights allows the storage and computation requirements of our algorithms to grow linearly in the problem size.

\subsection{Related Work}
\label{sec:related}

The literature on clustering is immense; the reader can consult the books \citep{Gor1999,Har1975,KauRou1990,Mir1996,WuWun2009} for a comprehensive review. The clustering function \Eqn{objective_function} can be viewed as a convex relaxation of either $k$-means clustering \citep{LinOhlLju2011} or hierarchical agglomerative clustering \citep{HocVerBac2011}. Both of these classical
clustering methods \citep{Sne1957,So1948,War1963} come in several varieties. The literature on $k$-means clustering
reports notable improvements in the computation \citep{Elk2003} and quality of solutions  \citep{ArtVas2007,BraManStr1997,KauRou1990} delivered by the standard greedy algorithms. Faster methods for agglomerative hierarchical clustering have been developed as well \citep{Fra1998}. Many statisticians view the hard cluster assignments of $k$-means as less desirable than the probabilistic assignments generated by mixture models \citep{McL2000,TitSmiMak1985}. Mixture models have the advantage of gracefully assigning points to overlapping clusters. These models are amenable to an EM algorithm and can be extended to infinite mixtures \citep{Fer1973,Ras2000,Nea2000}. 

Alternative approaches to clustering involve identifying components in the associated graph via its Laplacian matrix. Spectral clustering \citep{Lux2007} can be effective in cases when the clusters are non-convex and linearly inseparable. Although spectral clustering is valuable, it does not conflict with convex relaxation.  Indeed, \citet{HocVerBac2011} demonstrate that convex clustering can be effectively merged with spectral clustering. Although we agree with this point, the solution path uncovered by convex clustering is meritorious in its own right because it partially obviates the persistent need for determining the number of clusters. 

\subsection{Notation}
\label{sec:notation}

Throughout, scalars are denoted by lowercase letters ($a$), vectors
by boldface lowercase letters ($\bu$), and matrices by boldface capital letters ($\bU$). 
The $j$th column of a matrix $\bU$ is denoted by $\bu_{j}$. At times in our derivations, it will be easier to work with vectorized matrices. We adopt the convention of denoting the vectorization of a matrix $(\bU)$ by its lower case letter in boldface ($\bu$). Finally, we denote sets by upper case letters ($B$).

\subsection{Organization}
The rest of the paper is organized as follows. We first characterize the solution path theoretically. Previous papers take intuitive properties of the path for granted. We then review the ADMM and AMA algorithms and adapt them to solve the convex clustering problem. Once the algorithms are specified, we discuss their computational and storage complexity, convergence, and acceleration. We then present some numerical examples of clustering. The paper concludes with a general discussion.

\section{Properties of the solution path}
\label{sec:solutionpath}

The solution path $\bU(\gamma,\bw)$ has several nice properties as a function of the regularization parameter $\gamma$ and its weights $\bw = \{w_{ij}\}$ that expedite its numerical computation. The proof of the following two propositions can be found in the Supplemental Materials.

\begin{proposition}
\label{prop:solution_path_continuity} The solution path $\bU(\gamma)$ exists and depends continuously on $\gamma$. The path also depends continuously on the weight matrix $\bw$.
\end{proposition}

Existence and uniqueness of $\bU$ sets the stage for a well-posed optimization problem. 
Continuity of $\bU$ suggests employing homotopy continuation. Indeed, empirically we find great  time savings in solving a sequence of problems over a grid of $\gamma$ values when we use the solution of a previous value of $\gamma$ as a warm start or initial value for the next larger $\gamma$ value.

We also would like a rigorous argument that the centroids eventually coalesce to a common point as $\gamma$ becomes sufficiently large. For the example shown in \Fig{graph}, we intuitively expect for sufficiently large $\gamma$ that the
columns of $\bU$ satisfy $\bu_3 = \bu_4 = \bar{\bx}_{34}$ and $\bu_1 = \bu_2 = \bu_5 = \bar{\bx}_{125}$, where $\bar{\bx}_{34}$ is the mean of $\bx_3$ and $\bx_4$ and $\bar{\bx}_{125}$ is the mean of $\bx_1$, $\bx_2,$ and $\bx_5$. The next proposition confirms our intuition.

\begin{proposition}
\label{prop:coalesce}
Suppose each point corresponds to a node in a graph with an edge between nodes $i$ and $j$ whenever $w_{ij} > 0$. If this graph is connected, then $F_{\gamma}(\bU)$ is minimized by $\bar{\bX}$ for $\gamma$ sufficiently large, where each column of $\bar{\bX}$ equals the average $\bar{\bx}$ of the $n$ vectors $\bx_i$.
\end{proposition}

We close this section by noting that in general the clustering paths are not guaranteed to be agglomerative. In the special case of the $\ell_1$ norm with uniform weights $w_{ij} =1$, \citet{HocVerBac2011} prove that the path is agglomerative. In the same paper they give an $\ell_2$ norm example where the centroids fuse and then unfuse as the regularization parameter increases. This behavior, however, does not seem  to occur very frequently in practice. Nonetheless, in the algorithms we describe next, we allow for such fission events to ensure that our computed solution path is truly the global minimizer of the convex criterion (\ref{eq:objective_function}).

\section{Algorithms to Compute the Clustering Path}
\label{sec:algorithms}

Having characterized the solution path $\bU(\gamma)$, we now tackle the task of computing it. We present two closely related optimization approaches: the alternating direction method of multipliers (ADMM) \citep{BoyParChu2011,GabMer1976, GloMar1975} and the alternating minimization algorithm (AMA) \citep{Tse1991}. 
Both approaches employ variable splitting to handle the shrinkage penalties in the convex clustering criterion \Eqn{objective_function}.

\subsection{Reformulation of Convex Clustering}
\label{sec:reformulation}

Let us first recast the convex clustering problem as the equivalent constrained problem
\begin{equation}
\label{eq:split_objective_cluster}
\begin{split}
&\text{minimize} \; \frac{1}{2}\sum_{i=1}^n \|\bx_i-\bu_i\|_2^2 +
\gamma \sum_{l \in \mathcal{E}} w_{l} \|\bv_{l} \| \\
&\text{subject to} \; \bu_{l_1} - \bu_{l_2} - \bv_{l} = \bZero. 
\end{split}
\end{equation}
Here we index a centroid pair by $l = (l_1, l_2)$ with $l_1 < l_2$, define the set of edges over the non-zero weights $\mathcal{E} = \{ l = (l_1,l_2) : w_l > 0\}$, and introduce a new variable $\bv_l = \bu_{l_1} - \bu_{l_2}$ to account for the difference between the two centroids. The purpose of variable splitting is to simplify optimization with respect to the penalty terms.

Splitting methods such as ADMM and AMA have been successfully used to attack similar problems in image restoration \citep{GolOsh2009}. ADMM and AMA are now motivated as variants of the augmented Lagrangian method (ALM) \citep{Hes1969,NocWri2006,Pow1969,Roc1973}.  Let us review how ALM approaches the constrained optimization problem 
\begin{equation}
\label{eq:split_objective}
\begin{split}
&\text{minimize} \; f(\bu) + g(\bv) \\
&\text{subject to} \; \bA\bu + \bB\bv = \bc,
\end{split}
\end{equation}
which includes the constrained minimization problem \Eqn{split_objective_cluster} as a special case.
ALM solves the equivalent problem 
\begin{equation}
\begin{split}
\label{eq:split_objective_alm}
&\text{minimize}\; f(\bu) + g(\bv) + \frac{\nu}{2} \|  \bc - \bA\bu - \bB\bv \|_2^2, \\
&\text{subject to} \; \bA\bu + \bB\bv = \bc 
\end{split}
\end{equation}
by imposing a quadratic penalty on deviations from the feasible set. The two problems \Eqn{split_objective} and \Eqn{split_objective_alm}
are equivalent because their objective functions coincide for any point $(\bu, \bv)$ satisfying the equality constraint. We will see in a moment what the purpose of the quadratic penalty term is. First, recall that finding the minimizer to an equality constrained optimization problem is equivalent to the identifying the saddle point of the associated Lagrangian function.
The Lagrangian for the ALM problem \begin{eqnarray*}
\mathcal{L}_{\nu}(\bu,\bv,\blambda) & = & f(\bu) + g(\bv) + \langle \blambda, \bc - \bA\bu - \bB\bv \rangle
+ \frac{\nu}{2} \|  \bc - \bA\bu - \bB\bv \|_2^2 
\end{eqnarray*}
invokes the dual variable $\blambda$ as a vector of Lagrange multipliers. If $f(\bu)$ and $g(\bv)$ are convex and $\bA$ and $\bB$ have full column rank, then the objective \Eqn{split_objective_alm} is strongly convex, and the dual problem reduces to the unconstrained maximization of a concave function with Lipschitz continuous gradient. The dual problem is therefore a candidate for gradient ascent. In fact, this is the strategy that ALM takes in the updates
\begin{equation}
\label{eq:alm_updates}
\begin{split}
(\bu^{m+1}, \bv^{m+1}) & \amp = \amp \underset{\bu, \bv}{\arg\min}\; \mathcal{L}_\nu(\bu,\bv, \blambda^m) \\
\blambda^{m+1} & \amp = \amp \blambda^m + \nu (\bc - \bA\bu^{m+1} - \bB\bv^{m+1}). \\
\end{split}
\end{equation}

Unfortunately, the minimization of the augmented Lagrangian over $\bu$ and $\bv$ jointly is often difficult.  ADMM and AMA adopt different strategies in simplifying the minimization subproblem in the ALM
updates \Eqn{alm_updates}. ADMM minimizes the augmented Lagrangian one block of variables at a time. This
yields the algorithm
\begin{equation}
\label{eq:admm_updates}
\begin{split}
\bu^{m+1} & \amp = \amp \underset{\bu}{\arg\min}\; \mathcal{L}_\nu(\bu,\bv^m, \blambda^m) \\
\bv^{m+1} & \amp = \amp \underset{\bv}{\arg\min}\; \mathcal{L}_\nu(\bu^{m+1},\bv, \blambda^m) \\
\blambda^{m+1} & \amp = \amp \blambda^m + \nu(\bc - \bA\bu^{m+1} - \bB\bv^{m+1}).
\end{split}
\end{equation}
AMA takes a slightly different tack and updates the first block $\bu$ without augmentation,
assuming $f(\bu)$ is strongly convex. This change is accomplished by setting the positive tuning constant $\nu$ to be 0. Later we will see that this seemingly innocuous change will pay large dividends in the convex clustering problem. The overall algorithm iterates
according to 
\begin{equation}
\label{eq:ama_updates}
\begin{split}
\bu^{m+1} & \amp = \amp \underset{\bu}{\arg\min}\; \mathcal{L}_0(\bu,\bv^m, \blambda^m) \\
\bv^{m+1} & \amp = \amp \underset{\bv}{\arg\min}\; \mathcal{L}_\nu(\bu^{m+1},\bv, \blambda^m) \\
\blambda^{m+1} & \amp = \amp \blambda^m + \nu (\bc - \bA\bu^{m+1} - \bB\bv^{m+1}). 
\end{split}
\end{equation}
Although block descent appears to complicate matters, it often markedly simplifies optimization in the end. 
In the case of convex clustering, the updates are either simple linear transformations or evaluations of proximal maps.

\subsection{Proximal Map}
\label{sec:proximal}

For $\sigma > 0$ the function
\begin{eqnarray*}
\prox_{\sigma \Omega}(\bu) & = & \underset{\bv}{\arg\min}\;\left[\sigma \Omega(\bv)+ \frac{1}{2} \| \bu - \bv \|_2^2 \right]
\end{eqnarray*}
is a well-studied operation called the proximal map of the function $\Omega(\bv)$. The proximal map exists and is unique whenever the function $\Omega(\bv)$ is convex and lower semicontinuous. Norms satisfy these conditions,
and for many norms of interest the proximal map can be evaluated by either an explicit formula or an efficient algorithm. \Tab{prox} lists some common examples. The proximal maps for the $\ell_1$ and $\ell_2$ norms have explicit solutions and can be computed in $\mathcal{O}(p)$ operations for a vector $\bv \in \Real^p$. 
Another common example is the $\ell_{1,2}$ norm 
\begin{eqnarray*}
\|\bv\|_{1,2} & = & \sum_{g \in \mathcal{G}} \| \bv_g \|_2,
\end{eqnarray*}
which partitions the components of $\bv$ into non-overlapping groups $\mathcal{G}$. In this case there is also a simple shrinkage formula. The proximal map for the $\ell_\infty$ norm requires projection onto the unit simplex and lacks an explicit solution. However, there are good algorithms for projecting onto the unit simplex \citep{DucShaSin2008,Mic1986}. In particular, Duchi et al.\@'s projection algorithm makes it possible to evaluate $\prox_{\sigma\| \cdot \|_\infty}(\bv)$ in $\mathcal{O}(p\log p)$ operations. 

\begin{table}[th]
\caption[Proximal Map]{Proximal maps for common norms.
\label{tab:prox}}
\centering
\begin{tabular}{cccc}\\
\toprule
Norm & $\Omega(\bv)$ & $\prox_{\sigma\Omega}(\bv)$ & Comment \\
\midrule
$\ell_1$ & $\| \bv \|_1$ & $\left [ 1 - \frac{\sigma}{| v_l |} \right ]_+ v_l$ & Element-wise soft-thresholding \\
\midrule		
$\ell_2$ & $\| \bv \|_2$ & $\left [1 - \frac{\sigma}{\| \bv \|_2} \right]_+ \bv$ & Block-wise soft-thresholding \\
\midrule	
$\ell_{1,2}$ & $\sum_{g \in \mathcal{G}} \| \bv_g \|_2$ & $\left [1 - \frac{\sigma}{\| \bv_g \|_2} \right]_+ \bv_g$
& $\mathcal{G}$ is a partition of $\{1, \ldots, p\}$ \\
\midrule	
$\ell_\infty$ & $\| \bv \|_\infty$ & $\bv - \mathcal{P}_{\sigma S}(\bv)$ & $S$ is the unit simplex \\
\bottomrule
\end{tabular}
\end{table}

\subsection{ADMM updates}
\label{sec:ADMM}

The augmented Lagrangian is given by
\begin{equation}
\label{eq:augmented_Lagrangian}
\begin{split}
\mathcal{L}_\nu(\bU,\bV,\bLambda) & \amp = \amp \frac{1}{2}\sum_{i=1}^n \|\bx_i-\bu_i\|_2^2 +
\gamma \sum_{l \in \mathcal{E}} w_{l} \|\bv_{l} \| \\
& + \sum_{l \in \mathcal{E}} \langle \blambda_{l}, \bv_{l}-\bu_{l_1} +\bu_{l_2} \rangle
+ \,  \frac{\nu}{2}\, \sum_{l \in \mathcal{E}} \|\bv_{l}- \bu_{l_1} +\bu_{l_2} \|_2^2,
\end{split}
\end{equation}
where $\mathcal{E}$ is the set of edges corresponding to non-zero weights.  To update $\bU$ we need to minimize the following function
\begin{eqnarray*}
f(\bU) =  \frac{1}{2}\sum_{i=1}^n \|\bx_i-\bu_i\|_2^2 + \frac{\nu}{2}\, \sum_{l \in \mathcal{E}} \|\btildev_{l} - \bu_{l_1} +\bu_{l_2} \|_2^2,
\end{eqnarray*}
where $\btildev_{l} = \bv_{l} + \nu^{-1}\blambda_{l}$. We can rewrite the above function in terms of $\bu$ instead of the columns $\bu_i$ of the matrix $\bU$, namely
\begin{eqnarray*}
f(\bu) = \frac{1}{2} \lVert \bx - \bu \rVert_2^2 + \frac{\nu}{2}\sum_{l \in \mathcal{E}} \| \bA_{l} \bu - \btildev_{l} \|_2^2,
\end{eqnarray*}
where $\bA_{l} = \left [ (\be_{l_1} - \be_{l_2})^t \Kron \bI \right]$ and $\Kron$ denotes the Kronecker product. One can see this by noting that $\bu_{l_1} - \bu_{l_2} = \bU(\be_{l_1} - \be_{l_2})$ and applying the identity 
\begin{eqnarray}
\label{eq:vec}
\text{vec$(\bS\bT)$ = $[\bT^t \Kron \bI]$vec$(\bS)$.}
\end{eqnarray}
We can further simplify $f(\bu)$. If $\varepsilon = \lvert \mathcal{E} \rvert$ denotes the number of non-zero weights, then
\begin{eqnarray*}
f(\bu) = \frac{1}{2}\lVert \bu - \bx \rVert_2^2 + \frac{\nu}{2} \lVert \bA \bu - \btildev \rVert_2^2,
\end{eqnarray*}
where
\begin{eqnarray*}
\begin{gathered}
\bA^t \amp = \amp \begin{pmatrix}
\bA_{1}^t & \cdots & \bA_{\varepsilon}^t
\end{pmatrix}
\qtext{and}
\btildev^t \amp = \amp \begin{pmatrix}
\btildev_{1}^t & \cdots & \btildev_{\varepsilon} ^t
\end{pmatrix}.
\end{gathered}
\end{eqnarray*}

The stationary condition requires solving the linear system of equations
\begin{eqnarray*}
[\bI + \nu\bA^t\bA]\bu = \bx + \bA^t\btildev.
\end{eqnarray*}
The above system consists of $np$ equations in $np$ unknowns but has quite a bit of structure that we can exploit.
In fact, solving the above linear system is equivalent to solving a smaller system of $n$ equations in $n$ unknowns. Note that
\begin{eqnarray*}
\bI + \nu\bA^t\bA & = & \left [ \bI + \nu\sum_{l \in \mathcal{E}} (\be_{l_1} - \be_{l_2})(\be_{l_1} - \be_{l_2})^t \right ] \Kron \bI \\
\bA^t\btildev & = & \sum_{l \in \mathcal{E}} [(\be_{l_1} - \be_{l_2}) \Kron \bI]\btildev_{l}.
\end{eqnarray*}
Applying the above equalities, the identity (\ref{eq:vec}), and the fact that $[\bS \Kron \bT]^{-1} = \bS^{-1} \Kron \bT^{-1}$ when $\bS$ and $\bT$ are invertible gives the following equivalent linear system
\begin{eqnarray}
\label{eq:u_update_linear_system}
\bU \bM = \bX + \sum_{l \in \mathcal{E}} \btildev_{l} (\be_{l_1} - \be_{l_2})^t,
\end{eqnarray}
where
\begin{eqnarray*}
\bM = \bI + \nu\sum_{l \in \mathcal{E}} (\be_{l_1} - \be_{l_2})(\be_{l_1} - \be_{l_2})^t.
\end{eqnarray*}

If the edge set $\mathcal{E}$ contains all possible edges, then the update for $\bU$ can be computed analytically. The key observation is that in the completely connected case
\begin{eqnarray*}
\sum_{l \in \mathcal{E}} (\be_{l_1} - \be_{l_2})(\be_{l_1} - \be_{l_2})^t = n\bI - \bOne\bOne^t.
\end{eqnarray*}
Thus, the matrix $\bM$ can be expressed as the sum of a diagonal matrix and a rank-1 matrix, namely
\begin{eqnarray*}
\bM = (1 + n\nu)\bI - \nu \bOne\bOne^t.
\end{eqnarray*}
Applying the Sherman-Morrison formula, we can write the inverse of $\bM$ as
\begin{eqnarray*}
\bM^{-1} = \frac{1}{1 + n\nu} \left [\bI + \nu \bOne\bOne^t \right].
\end{eqnarray*}
Thus,
\begin{eqnarray*}
\bU = \frac{1}{1 + n\nu}
\left [\bX + \sum_{l \in \mathcal{E}} \btildev_{l} (\be_{l_1} - \be_{l_2})^t \right ]\left [\bI + \nu\bOne\bOne^t \right ].
\end{eqnarray*}
After some algebraic manipulations on the above equations, we arrive at the following updates
\begin{eqnarray*}
\bu_i & = & \frac{1}{1+n\nu} \by_i + \frac{n \nu}{1 + n\nu} \bar{\bx},
\end{eqnarray*}
where $\bar{\bx}$ is the average column of $\bX$ and
\begin{eqnarray*}
\by_i & = & \bx_i+ \sum_{l_1 = i} [\blambda_{l} + \nu \bv_{l} ]-\sum_{l_2 = i}[\blambda_{l} + \nu \bv_{l} ].
\end{eqnarray*}
Before deriving the updates for $\bV$, we remark that while using a fully connected weights graph allows us to write explicit updates for $\bU$, doing so comes at the cost of increasing the number of variables $\bv_l$ and $\blambda_l$. Such choices are not immaterial, and we will discuss these tradeoffs later in the paper.

To update $\bV$, we first observe that the Lagrangian $\mathcal{L}_\nu(\bU,\bV,\bLambda)$ is
separable in the vectors $\bv_{l}$. A particular difference vector $\bv_{l}$ is determined by the
proximal map
\begin{equation}
\label{eq:update_v}
\begin{split}
\bv_{l} & \amp = \amp \underset{\bv_l}{\arg\min} \; \frac{1}{2}  \left [\|\bv_l - (\bu_{l_1} - \bu_{l_2} - \nu^{-1}\blambda_{l})\|_2^2 + \frac{\gamma w_{l}}{\nu} \| \bv_l \| \right ] \\
 & \amp = \amp \prox_{\sigma_l \| \cdot \|}(\bu_{l_1} - \bu_{l_2} - \nu^{-1}\blambda_{l}),
\end{split}
\end{equation}
where $\sigma_l = \gamma w_l/\nu$. Finally, the Lagrange multipliers are updated by
\begin{eqnarray*}
\blambda_{l} & = & \blambda_{l} + \nu( \bv_{l}-\bu_{l_1}+\bu_{l_2}).
\end{eqnarray*}
\Alg{ADMM} summarizes the updates.

To track the progress of ADMM we use standard methods given in \citep{BoyParChu2011} based on primal and dual residuals. Details on the stopping rules that we employ are given in the Supplemental Materials.

\begin{algorithm}[t]
  \caption{ADMM}
  \label{alg:ADMM}
  Initialize $\bLambda^0$ and $\bV^0$.
  \begin{algorithmic}[1]
    \For{$m = 1, 2, 3, \ldots$} 
    	    \For{$i = 1, \ldots, n$}
	    		\State $\by_i = \bx_i + \sum_{l_1 = i} [\blambda^{m-1}_l + \nu \bv^{m-1}_l] 
			 - \sum_{l_2 = i} [\blambda^{m-1}_l + \nu \bv^{m-1}_l] $
	    \EndFor
	    \State $\bU^{m} = \frac{1}{1+n\nu} \bY + \frac{n\nu}{1 + n\nu} \bar{\bX}$
	    \ForAll{$l$}
		\State $\bv^m_l = \prox_{\sigma_l \| \cdot \|}(\bu^{m}_{l_1} - \bu^{m}_{l_2} - \nu^{-1}\blambda^{m-1}_{l})$
		\State $\blambda^m_{l} = \blambda^{m-1}_{l} + \nu( \bv^m_{l}-\bu^m_{l_1}+\bu^m_{l_2})$
	    \EndFor
    \EndFor
  \end{algorithmic}
\end{algorithm}

\subsection{AMA updates}
\label{sec:AMA}

Since AMA shares its update rules for $\bV$ and $\bLambda$ with ADMM, consider updating $\bU$.
Recall that AMA updates $\bU$ by minimizing the ordinary Lagrangian ($\nu =0$ case), namely
\begin{eqnarray*}
\bU^{m+1} & = & \underset{\bU}{\arg\min}\;
\frac{1}{2} \sum_{i=1}^n \| \bx_i - \bu_i \|_2^2 + \sum_{l} 
\langle \blambda^m_{l}, \bv_l - \bu_{l_1} + \bu_{l_2} \rangle. \\
\end{eqnarray*}
In contrast to ADMM, this minimization separates in each $\bu_i$ and gives an update that does not depend on $\bv_l$
\begin{eqnarray*}
\bu_i^{m+1} & = & \bx_i + \sum_{l_1 = i} \blambda^m_{l} - \sum_{l_2 = i} \blambda^m_{l}. \\
\end{eqnarray*}

Further scrutiny of the updates for $\bV$ and $\bLambda$ reveals additional simplifications.
Moreau's decomposition \citep{ComWaj2005}
\begin{eqnarray*}
\bz & = & \prox_{t h}(\bz)+t \prox_{t^{-1}h^\star}(t^{-1}\bz)
\end{eqnarray*}
allows one to express the proximal map of a function $h$ in terms of the proximal map of its Fenchel conjugate $h^\star$. This decomposition generalizes the familiar orthogonal projection decomposition, namely
$\bz = \mathcal{P}_W(\bz) + \mathcal{P}_{W^\perp}(\bz)$ where $W$ is a closed Euclidean subspace and $W^\perp$ is its orthogonal complement.
If $h(\bz) = \|\bz\|$ is a norm, then $h^\star(\bz) = \delta_B(\bz)$ is the convex indicator function of the
unit ball $B=\{ \by : \| \by \|_\dagger \leq 1\}$ of the dual norm $\| \cdot \|_\dagger$, namely the function that is 0 on $B$ and $\infty$ otherwise. Because the proximal map of the indicator function of a closed convex set collapses to projection onto the set, Moreau's decomposition leads to the identity
\begin{eqnarray}
\prox_{t h}(\bz) & = & \bz- t \prox_{t^{-1}\delta_B}(t^{-1}\bz) \amp = \amp \bz - t \mathcal{P}_{B}(t^{-1}\bz) \label{eq:Moreau_projection} \amp
= \amp \bz - \mathcal{P}_{tB}(\bz),
\end{eqnarray}
where $\mathcal{P}_B(\bz)$ denotes projection onto $B$. In this derivation the identity $t^{-1}\delta_B = \delta_B$ holds because $\delta_B$ takes only the values 0 and $\infty$. Applying the projection formula \Eqn{Moreau_projection}
 to the $\bv_l$ update \Eqn{update_v} yields the revised update
\begin{eqnarray*}
\bv^{m+1}_{l} & = & \bu^{m+1}_{l_1} - \bu^{m+1}_{l_2} - \nu^{-1}\blambda^{m}_l - \mathcal{P}_{tB}[\bu^{m+1}_{l_1}
-\bu^{m+1}_{l_2} - \nu^{-1}\blambda^m_l],
\end{eqnarray*}
for the constant $t = \sigma_l = \gamma w_l/\nu$.

The update for $\blambda_l$ is given by
\begin{eqnarray*}
\blambda^{m+1}_{l} & = & \blambda^{m}_{l} + \nu( \bv^{m+1}_{l}-\bu^{m+1}_{l_1}+\bu^{m+1}_{l_2}).
\end{eqnarray*}
Substituting for the above alternative expression for $\bv_l^{m+1}$ leads to substantial cancellations and the revised formula
\begin{eqnarray*}
\blambda^{m+1}_{l} & = & - \nu \mathcal{P}_{tB}[\bu^{m+1}_{l_1} - \bu^{m+1}_{l_2} - \nu^{-1}\blambda^m_{l}]. 
\end{eqnarray*}
The identities $-P_{tB}(\bz) = \mathcal{P}_{tB}(-\bz)$ and $a\mathcal{P}_{tB}(\bz) = \mathcal{P}_{atB}(a\bz)$ for $a > 0$ further simplify the update to
\begin{eqnarray*}
\blambda^{m+1}_{l} & = & \mathcal{P}_{C_l}(\blambda^{m}_{l} - \nu \bg^{m+1}_l),
\end{eqnarray*}
where $\bg^{m}_l = \bu^{m}_{l_1} - \bu^{m}_{l_2}, C_l = \{\blambda_l: \|\blambda_l\|_\dagger \le \gamma w_l\}$. \Alg{AMA} summarizes the AMA algorithm.  We highlight the fact that we no longer need to compute and store $\bv$ to perform the AMA updates. 

Note that the algorithm look remarkably like a projected gradient algorithm. Indeed, \citet{Tse1991} shows that AMA is actually performing proximal gradient ascent to maximize the dual problem. The dual of the convex clustering problem (\ref{eq:split_objective_cluster}) is 
\begin{equation}
 \label{eq:dual_function}
\begin{split}
D_\gamma(\bLambda) & \amp = \amp  \underset{\bU, \bV}\inf \; \mathcal{L}_{0}(\bU,\bV,\bLambda) \\
& \amp = \amp -\frac{1}{2} \sum_{i=1}^n \| \bDelta_i \|_2^2 - \sum_l \langle \blambda_l, \bx_{l_1} - \bx_{l_2} \rangle
- \sum_{l} \delta_{C_l}(\blambda_l),
\end{split}
\end{equation}
where
\begin{eqnarray*}
\bDelta_i & = & \sum_{l : l_1 = i} \blambda_l - \sum_{l : l_2 = i} \blambda_l .
\end{eqnarray*}
A derivation of the dual is given in the Supplemental Materials.

Since the dual is essentially a constrained least squares problem, it is hardly surprising that it can be solved numerically by the classic projected gradient algorithm. We will see in the next section that in addition to providing a simple interpretation of the AMA method, the dual allows us to derive a rigorous stopping criterion for AMA.  Before proceeding, however, let us emphasize 
that AMA requires tracking of only as many dual variable $\blambda_l$ as there are non-zero weights. We will find later that sparse weights often produces better quality clusterings. Thus, when relatively few weights are non-zero, the number of variables introduced by splitting does not become prohibitive under AMA.

\begin{algorithm}[t]
  \caption{AMA}
  \label{alg:AMA}
  Initialize $\blambda^0$.
  \begin{algorithmic}[1]
    \For{$m = 1, 2, 3, \ldots$} 
    	    \For{$i = 1, \ldots, n$}
	    \State $\bDelta^{m}_i = \sum_{l_1 = i} \blambda^{m-1}_l - \sum_{l_2 = i} \blambda_l^{m-1}$
	    \EndFor
	    \ForAll{$l$}
		\State $\bg_l^{m} = \bx_{l_1} - \bx_{l_2} + \bDelta_{l_1}^{m} - \bDelta_{l_2}^{m}$
		\State $\blambda^{m}_l = \mathcal{P}_{C_l}(\blambda^{m-1}_l - \nu \bg_l^{m})$
	    \EndFor
    \EndFor
  \end{algorithmic}
\end{algorithm}

\subsubsection{Stopping Criterion for AMA}

Recall that the duality gap at the $m$th iterate, $F_\gamma(\bU^{m}) - D_\gamma(\bLambda^m)$, is an upper bound on how far $F_\gamma(\bU^{m})$ is from the optimally minimal value of the objective function.  It is a certificate of optimality as there is a zero duality gap at an optimal solution. In short, if we can compute the duality gap, we can compute how suboptimal the last iterate is when the algorithm terminates.
The explicit functional forms \Eqn{split_objective_cluster} and \Eqn{dual_function}
of the primal and dual functions make it trivial to evaluate the duality
gap for feasible variables, since they depend on the quantities $\bDelta_i$ and $\bg_l  = \bu^{m}_{l_1} - \bu^{m}_{l_2}$, which are computed in the process of making the AMA updates. Thus, we stop the AMA algorithm when 
 \begin{eqnarray*}
F_\gamma(\bU^{m}) - D_\gamma(\bLambda^m) < \tau
\end{eqnarray*}
for $\tau>0$ small.

 \section{Convergence}
 \label{sec:convergence}

Both ADMM and AMA converge under reasonable conditions. Nonetheless, of the two, ADMM converges under broader conditions as its convergence is guaranteed for any $\nu > 0$. Convergence for AMA is guaranteed provided that $\nu$ is not too large. As we will see below, however, that bound is modest and easily identified in the convex clustering problem.

\subsection{AMA}
\label{sec:convergence_ama}

\cite{Tse1991} provides sufficient conditions to ensure the convergence of AMA. In the following list of assumptions, the functions $f(\bu)$ and $g(\bv)$ and parameters $\bA, \bB,$ and $\bc$ refer to problem \Eqn{split_objective}.
\begin{assumption}[Assumptions B and C in \cite{Tse1991}]
\label{as:Tseng}
\begin{itemize}
\item[(a)] $f(\bu)$ and $g(\bv)$ are convex lower-semicontinuous functions.
\item[(b)] $f(\bu)$ is strongly convex with modulus $\alpha > 0$.
\item[(c)] Problem \Eqn{split_objective} is feasible.
\item[(d)] The function $g(\bv) + \| \bB\bv \|_2^2$ has a minimum.
\item[(e)] The dual of \Eqn{split_objective} has an optimal Lagrange multiplier corresponding to the constraint 
$\bA\bu + \bB\bv = \bc$.
\end{itemize}
\end{assumption}
It is straightforward to verify that the functions and parameters in problem \Eqn{split_objective} satisfy \As{Tseng}. In particular, the strong convexity modulus $\alpha = 1$ for the convex clustering problem. In the derivation of the dual problem given in the Supplemental Materials, we briefly discuss how these assumptions are related to sufficient conditions for ensuring the convergence of the proximal gradient method applied to the dual problem. 
\begin{proposition}[Proposition 2 in \cite{Tse1991}]
\label{prop:ama_convergence}
Under \As{Tseng} the iterates generated by the AMA updates \Eqn{ama_updates}
satisfy the following:
\begin{itemize}
\item[(a)] $\lim_{m \to \infty} \bu^m = \bu^*$,
\item[(b)] $\lim_{m \to \infty} \bB\bv^m = \bc - \bA\bu^*$,
\item[(c)] $\lim_{m \to \infty} \blambda^m = \blambda^*$,
\end{itemize}
provided that $\nu < 2\alpha/\rho(\bA^t\bA)$, where $\rho(\bA^t\bA)$ denotes the largest eigenvalue of $\bA^t\bA$.
\end{proposition}

The parameter $\nu$ controlling the gradient step must be strictly less than twice the Lipschitz constant $1/\rho(\bA^t\bA)$.  To gain insight into how to choose $\nu$, let $\varepsilon \leq \binom{n}{2}$ denote the number of edges. Then $\bA = \bPhi \Kron \bI$, where $\bPhi$ is the $\varepsilon \times n$ oriented edge-vertex incidence matrix
\begin{eqnarray*}
\bPhi_{lv} = \begin{cases}
1 & \text{If node $v$ is the head of edge $l$} \\
-1 & \text{If node $v$ is the tail of edge $l$} \\
0 & \text{otherwise.}
\end{cases}
\end{eqnarray*}
Therefore, $\bA^t\bA = \bL \Kron \bI$, where $\bL = \bPhi^t\bPhi$ is the Laplacian matrix of the associated graph.  It is well known that the eigenvalues of $\bZ \Kron \bI$ coincide with the eigenvalues of $\bZ$. See for example Theorem~6 in Chapter 9 of \cite{Mil1987}. Therefore, $\rho(\bA^t\bA) = \rho(\bL)$. In lieu of computing $\rho(\bL)$ numerically, one can bound it by theoretical arguments. In general $\rho(\bL) \leq n$ \citep{AndMor1985}, with equality when the graph is fully connected and $w_{ij} > 0$ for all $i < j$.
Choosing a fixed step size of $\nu < 2/n$ works in practice when there are fewer than 1000 data points and the graph is dense. For a sparse graph with bounded node degrees, the sharper bound
\begin{eqnarray*} 
\rho(\bL) & \leq & \max\{ d(i) + d(j) : (i,j) \in \mathcal{E} \} \label{node_degree_bound}
\end{eqnarray*}
is available, where $d(i)$ is the degree of the $i$th node \citep{AndMor1985}. This bound can be computed quickly in $\mathcal{O}(n + \varepsilon)$ operations. Section \ref{timing_section} demonstrates the overwhelming speed advantage AMA on sparse graphs.

\subsection{ADMM}

Modest convergence results for the ADMM algorithm have been proven under minimal assumptions, which we now restate.
\begin{proposition}
\label{prop:ADMM_convergence}
If the functions $f(\bx)$ and $g(\bx)$ are closed, proper, and convex, and the unaugmented Lagrangian has a saddle point, then the ADMM iterates satisfy
\begin{eqnarray*}
\lim_{m \to \infty} \br^m & = & \bZero \\
\lim_{m \to \infty} \left[ f(\bU^m) + g(\bV^m) \right] & = & F^\star \\
\lim_{m \to \infty} \blambda^m & = & \blambda^*, 
\end{eqnarray*}
where $\br^m = \bc - \bA\bu^m - \bB\bv^m$ denotes the primal residuals and $F^\star$ denotes the minimal objective value of the primal problem.
\end{proposition}

Proofs of the above result can be found in the references \citep{BoyParChu2011,EckBer1992,Gab1983}. 
Note, however, the above results do not guarantee that the iterates $\bU^m$ converge to $\bU^\star$.
Since the
convex clustering criterion $F_\gamma(\bU)$ defined by equation \Eqn{objective_function} is strictly convex and coercive, we next show that we have the stronger result that the ADMM iterate sequence converges to the unique global minimizer $\bU^*$ of $F_\gamma(\bU)$.

\begin{proposition}
The iterates $\bU^m$ in \Alg{ADMM} converge to the unique global minimizer $\bU^*$ of the clustering
criterion $F_\gamma(\bU)$.
\end{proposition}
\begin{proof}
The conditions required by \Prop{ADMM_convergence} are obviously met by $F_\gamma(\bU)$. In particular, the unaugmented Lagrangian possesses a saddle point since the primal problem has a global minimizer. To validate
the conjectured limit, we first argue that the iterates $(\bU^m,\bV^m)$ are bounded. If on the contrary some subsequence is unbounded, then passing to the limit along this subsequence contradicts the limit 
\begin{eqnarray}
\lim_{m \to \infty} H_{\gamma}(\bU^m,\bV^m) & = & F_{\gamma}(\bU^*) \label{eq:fgamma_limit}
\end{eqnarray}
guaranteed by \Prop{ADMM_convergence} for the continuous function
\begin{eqnarray*}
H_{\gamma}(\bU,\bV) & = & \frac{1}{2}\sum_{i=1}^n \|\bx_i-\bu_i\|_2^2 + \gamma \sum_{l}w_{l} \| \bv_l \| .\end{eqnarray*}
To prove convergence of the sequence $(\bU^m,\bV^m)$, it therefore suffices to check that every limit point coincides with the minimum point of $F_\gamma(\bU)$. Let $(\bU^{m_n},\bV^{m_n})$ be a subsequence with limit $(\btildeU, \btildeV)$. According to \Prop{ADMM_convergence}, the differences $\bu^m_{l_1} - \bu^m_{l_2} - \bv^m_l$ tend to $\bZero$.  Thus, the limit $(\btildeU,\btildeV)$ is feasible.  Furthermore,
\begin{eqnarray*}
\lim_{n \to \infty} H_\gamma(\bU^{m_n},\bV^{m_n}) & = & H_\gamma(\btildeU,\btildeV) = F_\gamma(\btildeU).
\end{eqnarray*}
This limit contradicts the limit \Eqn{fgamma_limit} unless $F_\gamma(\btildeU) = F_\gamma(\bU^*)$. Because $\bU^*$ uniquely minimizes $F_\gamma(\bU)$, it follows that $\btildeU = \bU^*$. 
\end{proof}

\section{Acceleration}
\label{sec:acceleration}

Both AMA and ADMM admit acceleration at little additional computational cost. Given that AMA is a proximal gradient algorithm, \citet{GolODSet2012} show that it can be effectively accelerated via Nesterov's method \citep{BecTeb2009}. \Alg{AMA_fast} conveys the accelerated AMA method. \citet{GolODSet2012} also present methods for accelerating ADMM not considered in this paper.

\begin{algorithm}[t]
  \caption{Fast AMA}
  \label{alg:AMA_fast}
  Initialize $\blambda^{-1} = \btildelambda^{0}, \alpha_0 = 1$
  \begin{algorithmic}[1]
    \For{$m = 0, 1, 2, \ldots$} 
    	    \For{$i = 1, \ldots, n$}
	    \State $\bDelta^{m}_i = \sum_{l_1 = i} \blambda^{m-1}_l - \sum_{l_2 = i} \blambda_l^{m-1}$
	    \EndFor
	    \ForAll{$l$}
		\State $\bg_l^{m} = \bx_{l_1} - \bx_{l_2} + \bDelta_{l_1}^{m} - \bDelta_{l_2}^{m}$
		\State $\btildelambda^m_l = P_{C_l}(\blambda^{m-1}_l - \nu \bg_l^{m})$
	    \EndFor
	    \State $\alpha_{m} = (1 + \sqrt{1 + 4\alpha_{m-1}^2})/2$
	    	\State $\blambda^{m+1} = \btildelambda^{m} + \frac{\alpha_{m-1}}{\alpha_{m}}[\btildelambda^m - \btildelambda^{m-1}]$
    \EndFor
  \end{algorithmic}
\end{algorithm}

 \section{Computational Complexity}
 \label{sec:complexity}

\subsection{AMA}
In the sequel, we apply existing theory on the computational complexity of AMA to estimate the total number of iterations required by our AMA algorithm. The amount of work per iteration is specific to the variable splitting formulation of the clustering problem and depends on the sparsity of the matrix $\bA$ in the clustering problem. Suppose we wish to compute for a given $\gamma$ a solution such that the duality gap is at most $\tau$. We start by tallying the computational burden for a single round of AMA updates. Inspection of \Alg{AMA} shows that
computing all $\bDelta_i$ requires $p(2\varepsilon - n)$ total additions and subtractions. Computing all vectors
$\bg_l$ in \Alg{AMA} takes $\mathcal{O}(\varepsilon p)$ operations, and taking the subsequent gradient step also costs $\mathcal{O}(\varepsilon p)$ operations. Computing the needed projections costs $\mathcal{O}(\varepsilon p)$ operations for the $\ell_1$ and $\ell_2$ norms and $\mathcal{O}(\varepsilon p\log p)$ operations for the $\ell_\infty$ norm. Finally computing the duality gap costs $\mathcal{O}(np + \varepsilon p)$ operations. The assumption that $n$ is $\mathcal{O}(\varepsilon)$ entails smaller costs. A single iteration with gap checking then costs just $\mathcal{O}(\varepsilon p)$ operations for the $\ell_1$ and $\ell_2$ norms and 
$\mathcal{O}(\varepsilon p \log p)$ operations for the $\ell_\infty$ norm. 

Estimation of the number of iterations until convergence for proximal gradient descent and its Nesterov variant complete our analysis. The $np \times \varepsilon p$ matrix $\bA^t$ is typically short and fat. Consequently, the function $f^\star(\bA^t\blambda)$ is not strongly convex, 
and the best known convergence bounds for the proximal gradient method and its accelerated variant are sublinear \citep{BecTeb2009}. Specifically we have the following non-asymptotic bounds on the convergence of the objective values:
\begin{eqnarray*}
     D_\gamma(\blambda^*) - D_\gamma(\blambda^m) &  \le & \frac{\rho(\bA^t\bA) \|\blambda^* - \blambda^0\|_{2}^2}{2m}
\end{eqnarray*}
for the unaccelerated proximal gradient ascent and
\begin{eqnarray*}
     D_\gamma(\blambda^*) - D_\gamma(\blambda^m) &  \le & \frac{2 \rho(\bA^t\bA) \|\blambda^* - \blambda^0\|_{2}^2}{(m+1)^2},
\end{eqnarray*}
for its Nesterov accelerated alternative. Thus taking into account operations per iteration, we see that the  unaccelerated version and acceleration algorithms respectively require a computational effort of $\mathcal{O}(\frac{\varepsilon p}{\tau})$ and $\mathcal{O}(\frac{\varepsilon p}{\sqrt{\tau}})$ respectively for the $\ell_1$ and $\ell_2$ norms to attain a duality gap less than $\tau$. These bounds are respectively $\mathcal{O}(\frac{\varepsilon p\log p}{\tau})$ and $\mathcal{O}(\frac{\varepsilon p\log p)}{\sqrt{\tau}})$ for the $\ell_\infty$ norm. Total storage is $\mathcal{O}(p\varepsilon + np)$. In the worst case $\varepsilon$ is $\binom{n}{2}$. However, if we limit a node's connectivity to its $k$ nearest neighbors, then $\varepsilon$ is $\mathcal{O}(kn)$. Thus, the computational complexity of the problem in the worst case is quadratic in the number of points $n$ and linear under the restriction to $k$-nearest neighbors connectivity. The storage is quadratic in $n$ in the worst case and linear in $n$ under the $k$-nearest neighbors restriction. Thus, limiting a point's connectivity to its $k$-nearest neighbors renders both the storage requirements and operation counts linear in the problem size, namely $\mathcal{O}(knp)$.

\subsection{ADMM}
We have two cases to consider. First consider the explicit updates for outlined in \Alg{ADMM}, in which the edge set $\mathcal{E}$ contains every possible node pairing. By nearly identical arguments as above, the complexity of a single round of ADMM updates with primal and dual residual calculation requires $\mathcal{O}(n^2 p)$ operations for the $\ell_1$ and $\ell_2$ norms and $\mathcal{O}(n^2 p\log p)$ operations for the $\ell_\infty$ norm. Like AMA, it has been established that $\mathcal{O}(1/\tau)$ ADMM iterations are required to obtain an $\tau$-suboptimal solution \citep{HeYua2012}. Thus, the ADMM algorithm using explicit updates requires the same computational effort as AMA in its worst case, namely when all pairs of centroids are shrunk together. Moreover, the storage requirements are $\mathcal{O}(p n^2 + np)$.

The situation does not improve by much when we consider the more storage frugal alternative in which $\mathcal{E}$ contains only node pairings corresponding to non-zero weights. In this case, the variables $\bLambda$ and $\bV$ have only as many columns as there are non-zero weights. Now the storage requirements are $\mathcal{O}(p\varepsilon + np)$ like AMA, but the cost of updating $\bU$, the most computationally demanding step, remains quadratic in $n$. Recall we need to solve a linear system of equations (\ref{eq:u_update_linear_system})
\begin{eqnarray*}
\bU \bM = \bX + \sum_{l \in \mathcal{E}} \btildev_{l} (\be_{l_1} - \be_{l_2})^t,
\end{eqnarray*}
where $\bM \in \Real^{n \times n}$. Since $\bM$ is positive definite and does not change throughout the ADMM iterations, the prudent course of action is to compute and cache its Cholesky factorization. The factorization requires $\mathcal{O}(n^3)$ operations to calculate but that cost can be amortized across the repeated ADMM updates. With the Cholesky factorization in hand, we can update each row of $\bU$ by solving two sets of $n$-by-$n$ triangular systems of equations, which together requires $\mathcal{O}(n^2)$ operations. Since $\bU$ has $p$ rows, the total amount of work to update $\bU$ is $\mathcal{O}(n^2 p)$. Therefore, the overall amount of work per ADMM iteration is $\mathcal{O}(n^2 p + \varepsilon p)$ operations for the $\ell_1$ and $\ell_2$ norms and $\mathcal{O}(n^2 p + \varepsilon p\log p)$ operations for the $\ell_\infty$ norm. Thus, in stark constrast to AMA, both ADMM approaches grow quadratically, either in storage requirements or computational costs, regardless of how we might limit the size of the edge set $\mathcal{E}$.

\section{Practical Implementation}
\label{sec:practice}

This section addresses practical issues of algorithm implementation.

\subsection{Choosing weights}
\label{sec:weights}

The choice of the weights can dramatically affect the quality of the clustering path. We set the value of the weight between the $i$th and $j$th points to be $w_{ij} = \iota^k_{\{i,j\}} \exp(-\phi \| \bx_i - \bx_j \|_2^2)$, where $\iota^k_{\{i,j\}}$ is 1 if $j$ is among $i$'s $k$-nearest-neighbors or vice versa and 0 otherwise. The second factor is a Gaussian kernel that slows the coalescence of distant points. The constant $\phi$ is nonnegative; the value $\phi = 0$ corresponds to uniform weights.  As noted earlier, limiting positive weights to nearest neighbors improves both computational efficiency and clustering quality.  Although the two factors defining the weights act similarly, their combination increases the sensitivity of the clustering path to the local density of the data. 

\subsection{Making cluster assignments}
\label{sec:cluster_assignments}

We would like to be able to read off which centroids have fused as the regularization increases, namely determine clustering assignments as a function of $\gamma$. For both ADMM and AMA, such assignments can be performed in $\mathcal{O}(n)$ operations, using the differences variable $\bV$. In the case of AMA, where we do not store a running estimate of $\bV$, we compute $\bV$ using (\ref{eq:update_v}) after the algorithm terminates. In any case, once we have the variable $\bV$, we simply apply bread-first search to identify the connected components of the following graph induced by the $\bV$. The graph identifies a node with every data point and places an edge between the $l$th pair of points if and only if $\bv_l = \bZero$. Each connected component corresponds to a cluster. Note that the graph described here is a function of $\bV$ and is unrelated to the graph described earlier which is a function of the weights $w_{ij}$.

\section{Numerical Experiments}
\label{sec:experiments}

We now report numerical experiments on convex clustering for a synthetic data set and three real data sets. In particular, we focus on how the choice of the weights $w_{ij}$ affects the quality of the clustering solution. Prior research on this question is limited. Both Lindsten et al.\@ and Hocking et al.\@ suggest weights derived from Gaussian kernels and $k$-nearest neighbors. Because Hocking et al.\@ try only Gaussian kernels, in this section we follow up on their untested suggestion of combining Gaussian kernels and $k$-nearest neighbors.

We also compare the run times of our splitting methods to the run times of the subgradient algorithm employed by Hocking et al.\@ for $\ell_2$ paths. We focus our attention on solving the $\ell_2$ path since the rotational invariance of the 2-norm makes it a robust choice in practice. They provide R and C++ code for their algorithms. Our algorithms are implemented in R and C. To make a fair comparison,
we run our algorithm until it reaches a primal objective value that is less than or equal to the primal objective value obtained by the subgradient algorithm. To be specific, we first run the Hocking et al.\@ code to generate a clusterpath and record the sequence of $\gamma$'s generated by the Hocking et al.\@ code. We then run our algorithms over the same sequence of $\gamma$'s and stop once our primal objective value falls below those of Hocking et al.'s. We also keep the native stopping rule computations employed by our splitting methods, namely the dual loss calculations for AMA and residual calculations for ADMM. Since AMA already calculates the primal loss, this is not an additional burden. Although convergence monitoring creates additional work for ADMM, the added primal loss calculation at worst only changes the constant in the complexity bound. This follows since computing the primal loss requires $\mathcal{O}(np + \varepsilon p)$ operations to compute.

\subsection{Qualitative Comparisons}

Our next few examples demonstrate how the character of the solution paths can vary drastically with the choice of weights $w_{ij}$.

\subsubsection{Two Half Moons}
\label{sec:moons}

Consider the standard simulated data of two interlocking half moons in $\Real^2$ composed of 100 points each. \Fig{halfmoons} shows four convex clustering paths computed assuming two different numbers of nearest neighbors (10 and 50) and two different kernel constants $\phi$ (0 and 0.5).  The upper right panel makes it evident that limiting the number of nearest neighbors ($k=10$) and imposing non-uniform Gaussian kernel weights ($\phi=0.5$) produce the best clustering path. Using too many neighbors and assuming uniform weights results in little agglomerative clustering until late in the clustering path (lower left panel). The two intermediate cases diverge in interesting ways. The hardest set of points to cluster are the points in the upper half moon's right tip and the lower half moon's left tip.  Limiting the number of nearest neighbors and omitting the Gaussian kernel (upper left panel) correctly agglomerates the easier points, but waffles on the harder points, agglomerating them only at the very end when all points coalesce at the grand mean. Conversely, using too many neighbors and the Gaussian kernel (lower right panel) leads to a clustering path that does not hedge but incorrectly assigns the harder points.

\begin{figure}
\centering
\includegraphics[scale=0.45]{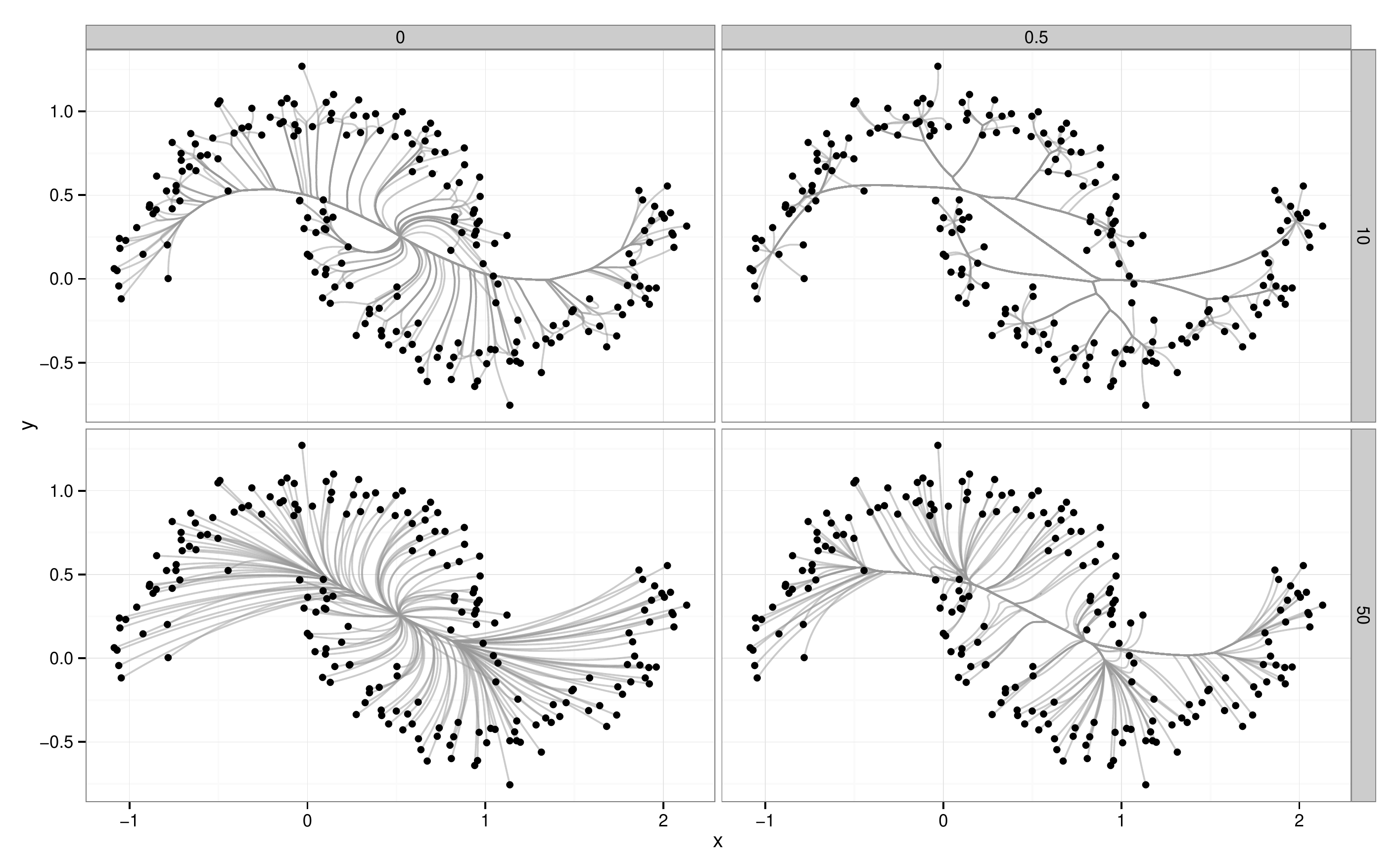}
\caption{Halfmoons Example: The first and second rows show results using $k=10$ and $50$ nearest neighbors respectively. The first and second columns show results using $\phi = 0$ and $0.5$ respectively.}
\label{fig:halfmoons}
\end{figure}

\subsubsection{Fisher's Iris Data} 
Fisher's Iris data \citep{Fis1936} consists of four measurements on 150 samples of iris flowers. There are three species present: setosa, versicolor, and virginica. \Fig{iris} shows the resulting clustering paths under two different choices of weights. On the left $w_{ij} = 1$ for all $i < j$, and on the right we use 5-nearest neighbors and $\phi = 4$.
Since there are four variables, to visualize results we project the data and the fitted clustering paths onto the first two principal components of the data. Again we see that more sensible clustering is observed when we choose weights to be sensitive to the local data density. We even get some separation between the overlapping species virginica and versicolor.

\subsubsection{Senate Voting} 
We consider Senate voting in 2001 on a subset of 15 issues selected by Americans for Democratic Action \citep{LeeMai2009,Dem2002}.
The data is binary. We limited our study to the 29 senators with unique voting records. The issues ranged over a wide spectrum: domestic, foreign, economic, military, environmental and social concerns. The final group of senators included 15 Democrats, 13 Republicans, and 1 Independent. \Fig{senate} shows the resulting clustering paths under two different choices of weights. On the left $w_{ij} = 1$ for all $i < j$, and on the right we use 15-nearest neighbors and $\phi = 0.5$. As observed previously, better clustering is observed when we choose the weights to be sensitive to the local data density. In particular, we get clear party separation. Note that we identify an outlying Democrat in Zel Miller and that the clustering seen agrees well with what PCA exposes.

\subsubsection{Dentition of mammals} Finally, we consider the problem of clustering mammals based on their dentition \citep{LeeMai2009,Har1975}. Eight different kinds of teeth are tallied up for each mammal: the number of top incisors, bottom incisors, top canines, bottom canines, top premolars, bottom premolars, top molars, and bottom molars. Again we removed observations with teeth distributions that were not unique, leaving us with 27 mammals. \Fig{mammals} shows the resulting clustering paths under two different choices of weights. On the left $w_{ij} = 1$ for all $i < j$, and on the right we use 5-nearest neighbors and $\phi = 0.5$. Once again, weights sensitive to the local density give superior results. In contrast to the iris and Senate data, the cluster path gives a different and perhaps more sensible solution than projections onto the first two components PCA. For example, the brown bat is considered more similar to the house bat and red bat, even though it is closer in the first two PCA coordinates to the coyote and oppossum.

\begin{figure}
\centering
\begin{tabular}{c}
\subfloat[Iris Data: Panel on the right (Set B) used $k=5$ nearest neighbors and $\phi=4$.]{\label{fig:iris}
\includegraphics[scale=0.2725]{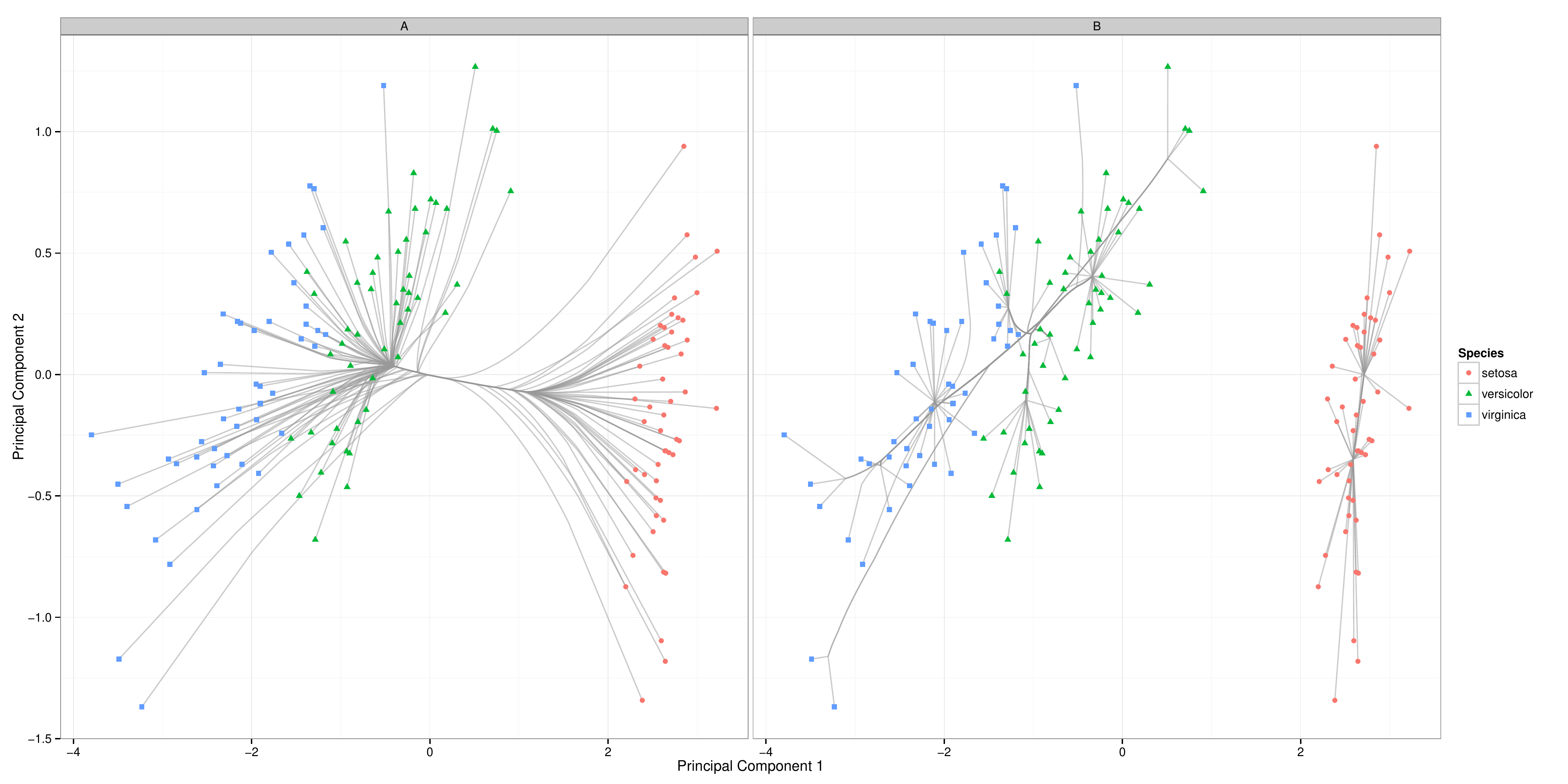}}\\
\subfloat[Senate: Panel on the right (Set B) used $k=15$ nearest neighbors and $\phi=0.5$.]{\label{fig:senate}
\includegraphics[scale=0.2725]{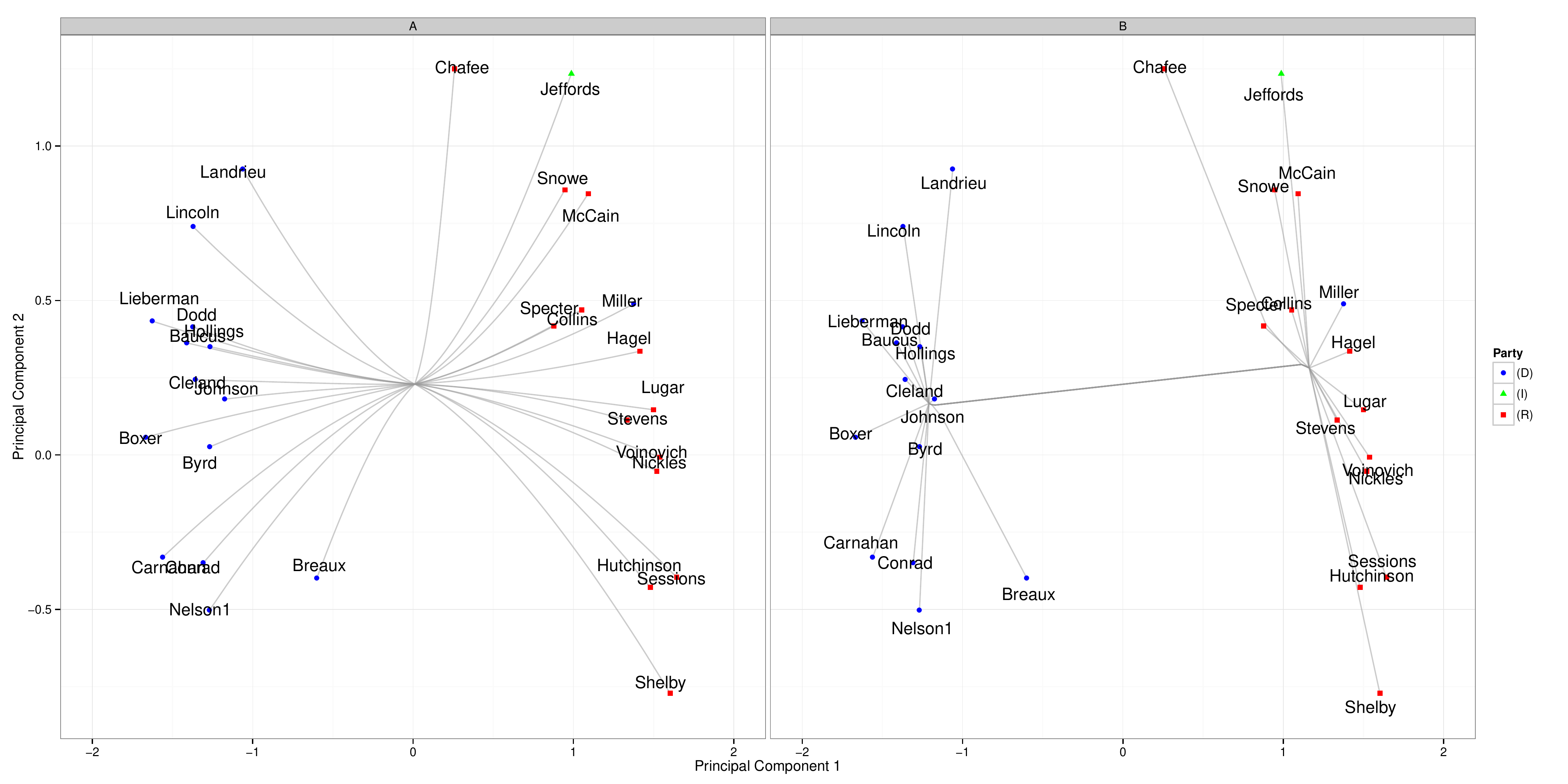}} \\
\subfloat[Mammal Data: Panel on the right (Set B) used $k=5$ nearest neighbors and $\phi=0.5$.]{\label{fig:mammals}
\includegraphics[scale=0.2625]{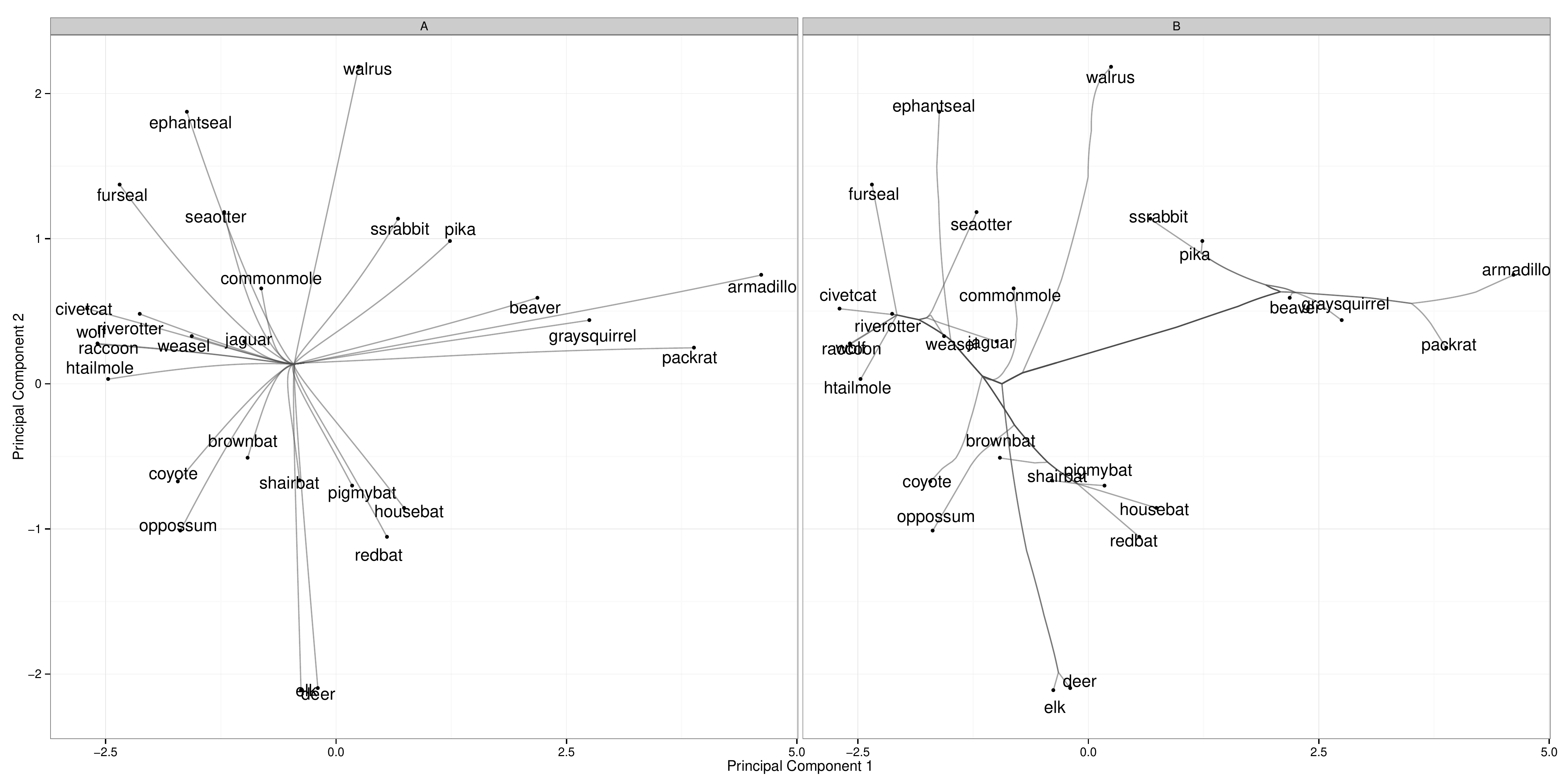}} \\
\end{tabular}
\caption{Clustering path under the $\ell_2$ norm. All panels on the left (Set A) used $w_{ij}=1$ for all $i < j$.}
\end{figure}

\subsection{Timing Comparisons \label{timing_section}}

We now present results on two batches of experiments, with dense weights in the first batch and sparse ones in the second. For the first set of experiments, we compared the run times of the subgradient descent algorithm of Hocking et al.\@, ADMM, and accelerated AMA on 10 replicates of simulated data consisting of 100, 200, 300, 400, and 500 points in $\Real^2$ drawn from a multivariate standard normal.  We limited our study to at most 500 points because the subgradient algorithm took several hours on a single realization of 500 points. Limiting the number of data points allowed us to use the simpler, but less storage efficient, ADMM formulation. For AMA, we fixed the step size at $\nu = 1/n$. For all tests, we assigned full-connectivity weights based on $\iota^k_{\{i,j\}} =1$ and $\phi = -2$. The parameter $\phi$ was chosen to ensure that the smallest weight was bounded safely away from zero. The full-connectivity assumption illustrates the superiority of AMA even under its least favorable circumstances. To trace out the entire clusterpath, we ran the Hocking subgradient algorithm to completion and invoked its default stopping criterion, namely a gradient with an $\ell_2$ norm below $0.001$. As noted earlier, we stopped our ADMM and AMA algorithms once their centroid iterates achieved a primal loss less than or equal to that achieved by the subgradient algorithm.

Table~\ref{tab:L2} shows the resulting mean times in seconds, and Figure~\ref{fig:times} shows box-plots of the square root of run time against the number of data points $n$. All three algorithms scale quadratically in the number of points. This is expected for ADMM and AMA because all weights $w_{ij}$ are positive. Nonetheless, the three algorithms possess different rate constants, with accelerated AMA possessing the slowest median growth, followed by the subgradient algorithm and ADMM. Again, to ensure fair comparisons with the subgradient algorithm, we required ADMM to make extra primal loss computations. This change tends to inflates its rate constant. Even so, we see that the spread in run times for the subgradient algorithm becomes very wide at 500 points, so much so that on some realizations even ADMM, with its additional overhead, is faster. In summary, we see that fast AMA leads to affordable computation times, on the order of minutes for hundreds of data points, in contrast to subgradient descent, which incurs run times on the order of hours for 400 to 500 data points.

In the second batch of experiments, the same set up is retained except for assignments of weights and step length choice for AMA. We used $\phi = -2$ again, but this time we zeroed out all weights except those corresponding to the $k = \frac{n}{4}$ nearest neighbors of each point. 
For AMA we used step sizes based on the bound (\ref{node_degree_bound}). Table~\ref{tab:L3} shows the resulting mean run times in seconds, and Figure~\ref{fig:times_sparse} shows box-plots of the square root of run time against the number of data points $n$. As attested by the shorter run times for all three algorithms, incorporation of sparse weights appears to make the problems easier to solve. Sparse weights also make ADMM competitive with the subgradient method for small to modest $n$. Even more noteworthy is the pronounced speed advantage of AMA over the other two algorithms for large $n$.  When clustering 500 points, AMA requires on average a mere 7 seconds compared to 6 to 7 minutes for the subgradient and ADMM algorithms.


\begin{table}[t]
  \centering
  \begin{tabular}{cccccc}
 & 100 & 200 & 300 & 400 & 500 \\ \hline
Subgradient & 44.40 & 287.86 &  2361.84 & 3231.21 & 13895.50\\
AMA & 16.09 & 71.67 & 295.23 & 542.45 & 1109.67 \\
ADMM & 109.13 & 922.72 & 3322.83 & 7127.22 & 13087.09 \\
  \end{tabular}
  \caption{Timing comparison under the $\ell_2$ norm: Dense weights. Mean run times are in seconds. Different methods are listed on each row. Each column reports times for varying number of points.}
  \label{tab:L2}
\end{table}

\begin{figure}
\centering
\includegraphics[scale=0.55]{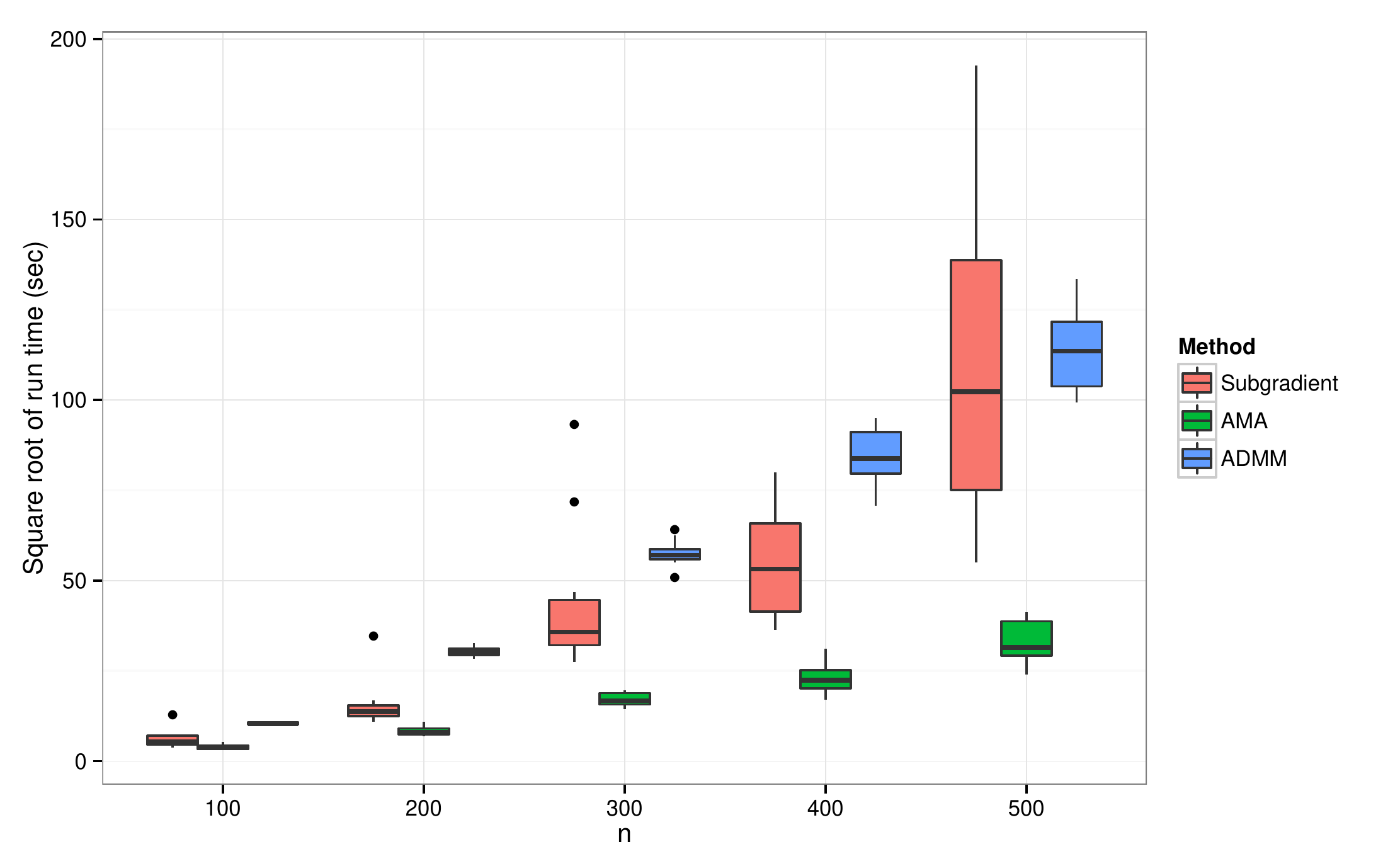}
\caption{Comparison of run times: Dense weights. The square root of the time is plotted against the number of points clustered.}
\label{fig:times}
\end{figure}

\begin{table}[t]
  \centering
  \begin{tabular}{cccccc}
 & 100 & 200 & 300 & 400 & 500 \\ \hline
Subgradient &  6.52 & 37.42 & 161.68 & 437.32 & 386.45 \\
AMA &  1.50 & 2.94 & 4.46 & 6.02 & 7.44 \\
ADMM &  5.42 & 30.93 & 88.63 & 192.54 & 436.49 \\
  \end{tabular}
  \caption{Timing comparison under the $\ell_2$ norm: Sparse weights. Mean run times are in seconds. Different methods are listed on each row. Each column reports times for varying number of points.}
  \label{tab:L3}
\end{table}

\begin{figure}
\centering
\includegraphics[scale=0.55]{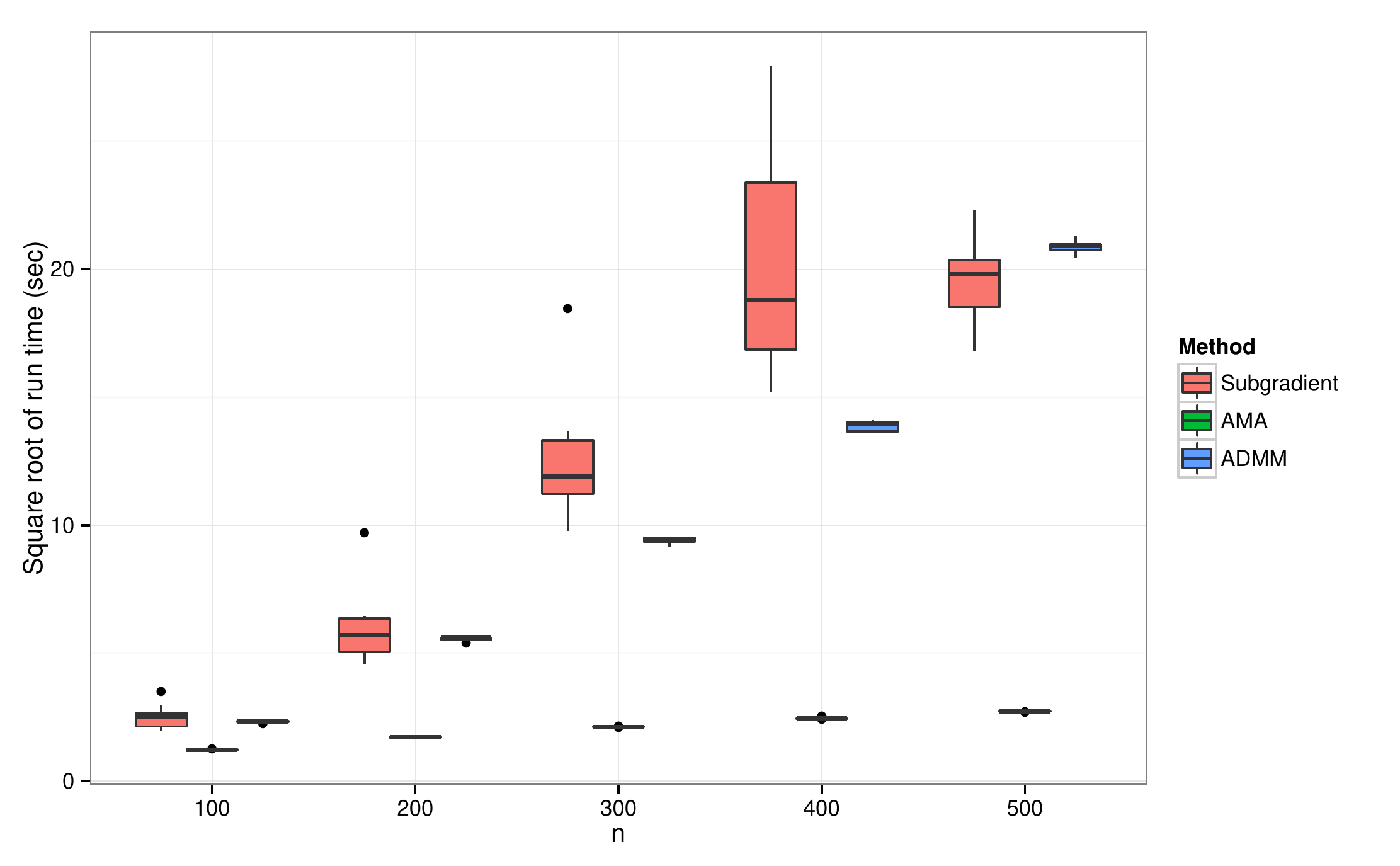}
\caption{Comparison of run times: Sparse weights. The square root of the time is plotted against the number of points clustered.}
\label{fig:times_sparse}
\end{figure}

\section{Conclusion \& Future Work}
\label{sec:conclusion}

In this paper, we introduce two splitting algorithms for solving the convex clustering problem. The splitting perspective encourages path following, one of the chief benefits of convex clustering. The splitting perspective also permits centroid penalties to invoke an arbitrary norm.  The only requirement is that the proximal map for the norm be readily computable. Equivalently, projection onto the unit ball of the dual norm should be straightforward. Because proximal maps and projection operators are generally well understood, it is possible for us to quantify the computational complexity and convergence properties of our algorithms.

It is noteworthy that ADMM did not fare as well as AMA. ADMM has become quite popular in machine learning circles in recent years. Applying variable splitting and using ADMM to iteratively solve 
the convex clustering problem seemed like an obvious and natural initial strategy. Only later during our study did we implement the less favored AMA algorithm. Considering how trivial the differences are between the generic block updates for ADMM (\ref{eq:admm_updates}) and AMA (\ref{eq:ama_updates}), we were surprised by the performance gap between them. In the convex clustering problem, however, there is a non-trivial difference between minimizing the augmented and unaugmented Lagrangian in the first block update. This task can be accomplished in less time and space by AMA.

Two features of the convex clustering problem make it an especially good candidate for solution by AMA. First, the objective function is strongly convex and therefore has a Lipschitz differentiable dual. Lipschitz differentiability is a standard condition ensuring the convergence of proximal gradient algorithms. For this reason \As{Tseng} (b) invokes strong convexity. Second, a good step size can be readily computed from the Laplacian matrix generated by the edge set $\mathcal{E}$. Without this prior bound, we would have to employ a more complicated line-search.

Our complexity analysis and simulations show that the accelerated AMA method appears to be the algorithm of choice. Nonetheless, given that alternative variants of ADMM may close the performance gap \citep{DenYin2012,GolMaSch2012}, we are reluctant to dismiss ADMM too quickly. Both algorithms deserve further investigation. For instance, in both ADMM and AMA, updates of $\bLambda$ and $\bV$ could be parallelized. Hocking et al.\@ also employed an active set approach to reduce computations as the centroids coalesce. A similar strategy could be adopted in our framework, but it incurs additional overhead as checks for fission events have to be introduced. An interesting and practical question brought up by Hocking et al.\@ remains open, namely under what conditions or weights are fusion events guaranteed to be permanent as $\gamma$ increases? In all our experiments, we did not observe any fission events. Identifying those conditions would eliminate the need to check for fission in such cases and expedite computation.

For AMA, the storage demands and computational complexity of convex clustering depend quadratically on the number of edges of the associated weight graph in the worst case. Limiting a point's connections to its $k$-nearest neighbors, for example, ensures that the number of edges in the graph is linear in the number of nodes in the graph. Eliminating long-range dependencies is often desirable anyway. Choosing sparse weights can improve both cluster quality and computational efficiency. 
Moreover, finding the exact $k$-nearest neighbors is likely not essential, and we conjecture that the quality of solutions would not suffer greatly if approximate nearest neighbors are used and algorithms for fast computation of approximately nearest neighbors are leveraged \citep{SlaCas2008}. On very large problems, the best strategy might be to exploit the continuity of solution paths in the weights. This suggests starting with even sparser graphs than the desired one and generating a sequence of solutions to increasingly dense problems. A solution with fewer edges can serve as a warm start for the next problem with more edges.

The splitting perspective also invites extensions that impose structured sparsity on the centroids. 
Witten and Tibshirani \citep{WitTib2010} discuss how sparse centroids can improve the quality of a solution, especially when only a relatively few features of data drive clustering. Structured sparsity can be accomplished by adding a sparsity-inducing norm penalty to the $\bU$ updates. The update for the centroids for both AMA and ADMM then rely on another proximal map of a gradient step. Introducing a sparsifying norm, however, raises the additional complication of choosing the amount of penalization.

Except for a few hints about weights, our analysis leaves the topic of optimal clustering untouched. Recently, \citet{Lux2010} suggested some principled approaches to assessing the quality of a clustering assignment via data perturbation and resampling. These clues are worthy of further investigation.

\appendix

\bigskip
\begin{center}
{\large\bf SUPPLEMENTAL MATERIALS}
\end{center}

\begin{description}

\item[Proofs and derivations:] The Supplemental Materials include proofs for Propositions~\ref{prop:solution_path_continuity} and \ref{prop:coalesce}, details on the stopping criterion for our ADMM algorithm, and the derivation of the dual function (\ref{eq:dual_function}). (Supplement.pdf)

\item[Code:] An R package, {\bf cvxclustr}, which implements the AMA and ADMM algorithms in this paper, is available on the CRAN website.

\end{description}

\begin{center}
{\large\bf ACKNOWLEDGMENTS}
\end{center}
The authors thank Jocelyn Chi, Daniel Duckworth, Tom Goldstein, Rob Tibshirani, and Genevera Allen for their helpful comments and suggestions. All plots were made
using the open source R package ggplot2 \citep{Wickham2009}. This research was supported by the NIH United States Public Health Service grants
GM53275 and HG006139

\bibliographystyle{asa}
\bibliography{SONCluster}

\begin{thebibliography}{66}
\newcommand{\enquote}[1]{``#1''}
\expandafter\ifx\csname natexlab\endcsname\relax\def\natexlab#1{#1}\fi

\bibitem[{Aloise et~al.(2009)Aloise, Deshpande, Hansen, and
  Popat}]{AloDesHan2009}
Aloise, D., Deshpande, A., Hansen, P., and Popat, P. (2009),
  \enquote{{NP}-hardness of {E}uclidean sum-of-squares clustering,}
  \textit{Machine Learning}, 75, 245--248.

\bibitem[{{Americans for Democratic Action}(2002)}]{Dem2002}
{Americans for Democratic Action} (2002), \enquote{2001 voting record:
  Shattered promise of liberal progress,} \textit{ADA Today}, 57, 1--17.

\bibitem[{Anderson and Morley(1985)}]{AndMor1985}
Anderson, W.~N. and Morley, T.~D. (1985), \enquote{Eigenvalues of the
  {L}aplacian of a graph,} \textit{Linear and Multilinear Algebra}, 18,
  141--145.

\bibitem[{Arthur and Vassilvitskii(2007)}]{ArtVas2007}
Arthur, D. and Vassilvitskii, S. (2007), \enquote{k-means++: {T}he advantages
  of careful seeding,} in \textit{Proceedings of the eighteenth annual ACM-SIAM
  symposium on Discrete algorithms}, Philadelphia, PA, USA: Society for
  Industrial and Applied Mathematics, SODA '07, pp. 1027--1035.

\bibitem[{Beck and Teboulle(2009)}]{BecTeb2009}
Beck, A. and Teboulle, M. (2009), \enquote{A Fast Iterative
  Shrinkage-Thresholding Algorithm for Linear Inverse Problems,} \textit{SIAM
  Journal on Imaging Sciences}, 2, 183--202.

\bibitem[{Boyd et~al.(2011)Boyd, Parikh, Chu, Peleato, and
  Eckstein}]{BoyParChu2011}
Boyd, S., Parikh, N., Chu, E., Peleato, B., and Eckstein, J. (2011),
  \enquote{Distributed Optimization and Statistical Learning via the
  Alternating Direction Method of Multipliers,} \textit{Found. Trends Mach.
  Learn.}, 3, 1--122.

\bibitem[{Bradley et~al.(1997)Bradley, Mangasarian, and Street}]{BraManStr1997}
Bradley, P.~S., Mangasarian, O.~L., and Street, W.~N. (1997),
  \enquote{Clustering via Concave Minimization,} in \textit{Advances in Neural
  Information Processing Systems}, MIT Press, pp. 368--374.

\bibitem[{Chen et~al.(1998)Chen, Donoho, and Saunders}]{CheDonSau1998}
Chen, S.~S., Donoho, D.~L., and Saunders, M.~A. (1998), \enquote{Atomic
  Decomposition by Basis Pursuit,} \textit{SIAM Journal on Scientific
  Computing}, 20, 33--61.

\bibitem[{Combettes and Wajs(2005)}]{ComWaj2005}
Combettes, P. and Wajs, V. (2005), \enquote{Signal Recovery by Proximal
  Forward-Backward Splitting,} \textit{Multiscale Modeling \& Simulation}, 4,
  1168--1200.

\bibitem[{{CVX Research, Inc.}(2012)}]{CVX2012}
{CVX Research, Inc.} (2012), \enquote{{CVX}: Matlab Software for Disciplined
  Convex Programming, version 2.0 beta,} \url{http://cvxr.com/cvx}.

\bibitem[{Dasgupta and Freund(2009)}]{DasFre2009}
Dasgupta, S. and Freund, Y. (2009), \enquote{Random projection trees for vector
  quantization,} \textit{IEEE Trans. Inf. Theor.}, 55, 3229--3242.

\bibitem[{de~Leeuw and Mair(2009)}]{LeeMai2009}
de~Leeuw, J. and Mair, P. (2009), \enquote{Gifi Methods for Optimal Scaling in
  {R: The Package homals},} \textit{Journal of Statistical Software}, 31,
  1--21.

\bibitem[{Deng and Yin(2012)}]{DenYin2012}
Deng, W. and Yin, W. (2012), \enquote{On the Global and Linear Convergence of
  the Generalized Alternating Direction Method of Multipliers,} Tech. Rep. CAAM
  Technical Report TR12-14, Rice University.

\bibitem[{Duchi et~al.(2008)Duchi, Shalev-Shwartz, Singer, and
  Chandra}]{DucShaSin2008}
Duchi, J., Shalev-Shwartz, S., Singer, Y., and Chandra, T. (2008),
  \enquote{Efficient Projections onto the $\ell_1$-Ball for Learning in High
  Dimensions,} in \textit{Proceedings of the International Conference on
  Machine Learning}.

\bibitem[{Eckstein and Bertsekas(1992)}]{EckBer1992}
Eckstein, J. and Bertsekas, D.~P. (1992), \enquote{On the {Douglas-Rachford}
  splitting method and the proximal point algorithm for maximal monotone
  operators,} \textit{Mathematical Programming}, 55, 293--318.

\bibitem[{Elkan(2003)}]{Elk2003}
Elkan, C. (2003), \enquote{Using the Triangle Inequality to Accelerate
  $k$-Means,} in \textit{Proceedings of ICML 2003}.

\bibitem[{Ferguson(1973)}]{Fer1973}
Ferguson, T.~S. (1973), \enquote{A {B}ayesian Analysis of Some Nonparametric
  Problems,} \textit{Annals of Statistics}, 1, 209--230.

\bibitem[{Fisher(1936)}]{Fis1936}
Fisher, R.~A. (1936), \enquote{The Use of Multiple Measurements in Taxonomic
  Problems,} \textit{Annals of Eugenics}, 7, 179--188.

\bibitem[{Forgy(1965)}]{For1965}
Forgy, E. (1965), \enquote{Cluster analysis of multivariate data: efficiency
  versus interpretability of classifications,} \textit{Biometrics}, 21,
  768--780.

\bibitem[{Fraley(1998)}]{Fra1998}
Fraley, C. (1998), \enquote{Algorithms for Model-Based Gaussian Hierarchical
  Clustering,} \textit{SIAM Journal on Scientific Computing}, 20, 270--281.

\bibitem[{Frank and Wolfe(1956)}]{FraWol1956}
Frank, M. and Wolfe, P. (1956), \enquote{An algorithm for quadratic
  programming,} \textit{Naval Research Logistics Quarterly}, 3, 95--110.

\bibitem[{Gabay(1983)}]{Gab1983}
Gabay, D. (1983), \textit{Augmented Lagrangian Methods: Applications to the
  Solution of Boundary-Value Problems}, Amsterdam: North-Holland, chap.
  Applications of the method of multipliers to variational inequalities.

\bibitem[{Gabay and Mercier(1976)}]{GabMer1976}
Gabay, D. and Mercier, B. (1976), \enquote{A dual algorithm for the solution of
  nonlinear variational problems via finite-element approximations,}
  \textit{Comp. Math. Appl.}, 2, 17--40.

\bibitem[{Glowinski and Marrocco(1975)}]{GloMar1975}
Glowinski, R. and Marrocco, A. (1975), \enquote{Sur lapproximation par elements
  finis dordre un, et la resolution par penalisation-dualite dune classe de
  problemes de Dirichlet nonlineaires,} \textit{Rev. Francaise dAut. Inf. Rech.
  Oper.}, 2, 41--76.

\bibitem[{Goldfarb et~al.(2012)Goldfarb, Ma, and Scheinberg}]{GolMaSch2012}
Goldfarb, D., Ma, S., and Scheinberg, K. (2012), \enquote{Fast alternating
  linearization methods for minimizing the sum of two convex functions,}
  \textit{Mathematical Programming}, 1--34.

\bibitem[{Goldstein et~al.(2012)Goldstein, O'Donoghue, and
  Setzer}]{GolODSet2012}
Goldstein, T., O'Donoghue, B., and Setzer, S. (2012), \enquote{Fast Alternating
  Direction Optimization Methods,} Tech. Rep. cam12-35, University of
  California, Los Angeles.

\bibitem[{Goldstein and Osher(2009)}]{GolOsh2009}
Goldstein, T. and Osher, S. (2009), \enquote{The Split {Bregman} Method for
  L1-Regularized Problems,} \textit{SIAM Journal on Imaging Sciences}, 2,
  323--343.

\bibitem[{Gordon(1999)}]{Gor1999}
Gordon, A. (1999), \textit{Classification}, London: Chapman and Hall/CRC Press,
  2nd ed.

\bibitem[{Gower and Ross(1969)}]{GowRos1969}
Gower, J.~C. and Ross, G. J.~S. (1969), \enquote{Minimum Spanning Trees and
  Single Linkage Cluster Analysis,} \textit{Journal of the Royal Statistical
  Society. Series C (Applied Statistics)}, 18, 54--64.

\bibitem[{Grant and Boyd(2008)}]{GraBoy2008}
Grant, M. and Boyd, S. (2008), \enquote{Graph implementations for nonsmooth
  convex programs,} in \textit{Recent Advances in Learning and Control}, eds.
  Blondel, V., Boyd, S., and Kimura, H., Springer-Verlag Limited, Lecture Notes
  in Control and Information Sciences, pp. 95--110,
  \url{http://stanford.edu/~boyd/graph_dcp.html}.

\bibitem[{Hartigan(1975)}]{Har1975}
Hartigan, J. (1975), \textit{Clustering Algorithms}, New York: Wiley.

\bibitem[{He and Yuan(2012)}]{HeYua2012}
He, B. and Yuan, X. (2012), \enquote{On the $O(1/n)$ Convergence Rate of the
  DouglasÐRachford Alternating Direction Method,} \textit{SIAM Journal on
  Numerical Analysis}, 50, 700--709.

\bibitem[{Hestenes(1969)}]{Hes1969}
Hestenes, M. (1969), \enquote{Multiplier and gradient methods,} \textit{Journal
  of Optimization Theory and Applications}, 4, 303--320.

\bibitem[{Hocking et~al.(2011)Hocking, Vert, Bach, and Joulin}]{HocVerBac2011}
Hocking, T., Vert, J.-P., Bach, F., and Joulin, A. (2011),
  \enquote{Clusterpath: an Algorithm for Clustering using Convex Fusion
  Penalties,} in \textit{Proceedings of the 28th International Conference on
  Machine Learning (ICML-11)}, eds. Getoor, L. and Scheffer, T., New York, NY,
  USA: ACM, ICML '11, pp. 745--752.

\bibitem[{Hoefling(2010)}]{Hoe2010}
Hoefling, H. (2010), \enquote{A Path Algorithm for the Fused Lasso Signal
  Approximator,} \textit{Journal of Computational and Graphical Statistics},
  19, 984--1006.

\bibitem[{Johnson(1967)}]{Joh1967}
Johnson, S. (1967), \enquote{Hierarchical clustering schemes,}
  \textit{Psychometrika}, 32, 241--254.

\bibitem[{Kaufman and Rousseeuw(1990)}]{KauRou1990}
Kaufman, L. and Rousseeuw, P. (1990), \textit{Finding Groups in Data: An
  Introduction to Cluster Analysis}, New York: Wiley.

\bibitem[{Lance and Williams(1967)}]{LanWil1967}
Lance, G.~N. and Williams, W.~T. (1967), \enquote{A General Theory of
  Classificatory Sorting Strategies: 1. Hierarchical Systems,} \textit{The
  Computer Journal}, 9, 373--380.

\bibitem[{Lindsten et~al.(2011)Lindsten, Ohlsson, and Ljung}]{LinOhlLju2011}
Lindsten, F., Ohlsson, H., and Ljung, L. (2011), \enquote{Just Relax and Come
  Clustering! {A} ConvexiÞcation of k-Means Clustering,} Tech. rep.,
  Link\"{o}pings universitet.

\bibitem[{Lloyd(1982)}]{Llo1982}
Lloyd, S. (1982), \enquote{Least squares quantization in {PCM},}
  \textit{Information Theory, IEEE Transactions on}, 28, 129 -- 137.

\bibitem[{Luxburg(2007)}]{Lux2007}
Luxburg, U. (2007), \enquote{A tutorial on spectral clustering,}
  \textit{Statistics and Computing}, 17, 395--416.

\bibitem[{MacQueen(1967)}]{Mac1967}
MacQueen, J. (1967), \enquote{Some methods for classification and analysis of
  multivariate observations,} in \textit{Proc. Fifth Berkeley Symp. on Math.
  Statist. and Prob.}, Univ. of Calif. Press, vol.~1, pp. 281--297.

\bibitem[{McLachlan(2000)}]{McL2000}
McLachlan, G. (2000), \textit{Finite Mixture Models}, Hoboken, New Jersey:
  Wiley.

\bibitem[{Michelot(1986)}]{Mic1986}
Michelot, C. (1986), \enquote{{A finite algorithm for finding the projection of
  a point onto the canonical simplex of $\Real^n$},} \textit{Journal of
  Optimization Theory and Applications}, 50, 195--200.

\bibitem[{Miller(1987)}]{Mil1987}
Miller, K.~S. (1987), \textit{Some eclectic matrix theory}, {Robert E. Krieger
  Publishing Company, Inc.}

\bibitem[{Mirkin(1996)}]{Mir1996}
Mirkin, B. (1996), \textit{Mathematical Classification and Clustering},
  Dordrecht, The Netherlands: Kluwer Academic Publishers.

\bibitem[{Murtagh(1983)}]{Mur1983}
Murtagh, F. (1983), \enquote{A Survey of Recent Advances in Hierarchical
  Clustering Algorithms,} \textit{The Computer Journal}, 26, 354--359.

\bibitem[{Neal(2000)}]{Nea2000}
Neal, R.~M. (2000), \enquote{Markov Chain Sampling Methods for {D}irichlet
  Process Mixture Models,} \textit{Journal of Computational and Graphical
  Statistics}, 9, pp. 249--265.

\bibitem[{Nocedal and Wright(2006)}]{NocWri2006}
Nocedal, J. and Wright, S. (2006), \textit{Numerical Optimization}, Springer,
  2nd ed.

\bibitem[{Powell(1969)}]{Pow1969}
Powell, M. (1969), \enquote{A method for nonlinear constraints in minimization
  problems,} in \textit{Optimization}, ed. Fletcher, R., Academic Press, pp.
  283--298.

\bibitem[{Rasmussen(2000)}]{Ras2000}
Rasmussen, C.~E. (2000), \enquote{The Infinite {G}aussian Mixture Model,} in
  \textit{In Advances in Neural Information Processing Systems 12}, MIT Press,
  pp. 554--560.

\bibitem[{Rockafellar(1973)}]{Roc1973}
Rockafellar, R. (1973), \enquote{The multiplier method of {H}estenes and
  {P}owell applied to convex programming,} \textit{Journal of Optimization
  Theory and Applications}, 12, 555--562.

\bibitem[{Slaney and Casey(2008)}]{SlaCas2008}
Slaney, M. and Casey, M. (2008), \enquote{Locality-Sensitive Hashing for
  Finding Nearest Neighbors [Lecture Notes],} \textit{Signal Processing
  Magazine, IEEE}, 25, 128--131.

\bibitem[{Sneath(1957)}]{Sne1957}
Sneath, P. H.~A. (1957), \enquote{The Application of Computers to Taxonomy,}
  \textit{Journal of General Microbiology}, 17, 201--226.

\bibitem[{S{\o}rensen(1948)}]{So1948}
S{\o}rensen, T. (1948), \enquote{A method of establishing groups of equal
  amplitude in plant sociology based on similarity of species and its
  application to analyses of the vegetation on Danish commons,}
  \textit{Biologiske Skrifter}, 5, 1--34.

\bibitem[{Tibshirani(1996)}]{Tib1996}
Tibshirani, R. (1996), \enquote{Regression Shrinkage and Selection via the
  Lasso,} \textit{Journal of the Royal Statistical Society, Ser. B}, 58,
  267--288.

\bibitem[{Tibshirani et~al.(2005)Tibshirani, Saunders, Rosset, Zhu, and
  Knight}]{TibSauRos2005}
Tibshirani, R., Saunders, M., Rosset, S., Zhu, J., and Knight, K. (2005),
  \enquote{Sparsity and smoothness via the fused lasso,} \textit{Journal of the
  Royal Statistical Society: Series B (Statistical Methodology)}, 67, 91--108.

\bibitem[{Tibshirani and Taylor(2011)}]{TibTay2011}
Tibshirani, R.~J. and Taylor, J. (2011), \enquote{The solution path of the
  generalized lasso,} \textit{Annals of Statistics}, 39, 1335--1371.

\bibitem[{Titterington et~al.(1985)Titterington, Smith, and
  Makov}]{TitSmiMak1985}
Titterington, D.~M., Smith, A. F.~M., and Makov, U.~E. (1985),
  \textit{Statistical Analysis of Finite Mixture Distributions}, Hoboken, New
  Jersey: John Wiley \& Sons.

\bibitem[{Tropp(2006)}]{Tro2006}
Tropp, J. (2006), \enquote{Just relax: {C}onvex programming methods for
  identifying sparse signals in noise,} \textit{IEEE Transactions on
  Information Theory}, 52, 1030 --1051.

\bibitem[{Tseng(1991)}]{Tse1991}
Tseng, P. (1991), \enquote{Applications of a Splitting Algorithm to
  Decomposition in Convex Programming and Variational Inequalities,}
  \textit{SIAM Journal on Control and Optimization}, 29, 119--138.

\bibitem[{von Luxburg(2010)}]{Lux2010}
von Luxburg, U. (2010), \enquote{Clustering Stability: {An} Overview,}
  \textit{Found. Trends Mach. Learn.}, 2, 235--274.

\bibitem[{Ward(1963)}]{War1963}
Ward, J.~H. (1963), \enquote{Hierarchical Grouping to Optimize an Objective
  Function,} \textit{Journal of the American Statistical Association}, 58,
  236--244.

\bibitem[{Wickham(2009)}]{Wickham2009}
Wickham, H. (2009), \textit{ggplot2: Elegant Graphics for Data Analysis},
  Springer New York.

\bibitem[{Witten and Tibshirani(2010)}]{WitTib2010}
Witten, D.~M. and Tibshirani, R. (2010), \enquote{A framework for feature
  selection in clustering,} \textit{J Am Stat Assoc.}, 105, 713--726.

\bibitem[{Wu and Wunsch(2009)}]{WuWun2009}
Wu, R. and Wunsch, D. (2009), \textit{Clustering}, Hoboken: John Wiley \& Sons.

\end{thebibliography}

\end{document}